\documentclass[12pt]{l4dc2021}

\usepackage{times}

\usepackage{hyperref}%
\hypersetup{colorlinks=true, linkcolor=blue, breaklinks=true, urlcolor=blue}
\usepackage{doi}

\usepackage[all]{xy}
\usepackage{amsmath}
\usepackage{amscd}
\usepackage{mathrsfs}
\usepackage{amssymb}
\usepackage[utf8]{inputenc}
\usepackage[yyyymmdd]{datetime}
\usepackage{graphicx}
\usepackage{floatrow}
\usepackage{lipsum}
\usepackage{multicol} \usepackage{mathtools}
\usepackage{multirow}
\usepackage{threeparttable}
\usepackage{stmaryrd}
\usepackage{enumitem,sidecap}
\usepackage{wrapfig}

\usepackage[font=small]{caption}

\usepackage{multicol} \usepackage{mathtools}
\usepackage{version,xspace}
\usepackage{url,doi}
\newcommand{\until}[1]{\{1,\dots, #1\}}

\newcommand{\supscr}[2]{#1^{\textup{#2}}} \newcommand{\setdef}[2]{\{#1
	\; | \; #2\}}

\newcommand{\map}[3]{#1: #2 \rightarrow #3}

\newcommand{\WSIP}[2]{\left\llbracket{#1}, {#2}\right\rrbracket}

\newcommand\oprocendsymbol{\hbox{$\triangle$}}
\newcommand\oprocend{\relax\ifmmode\else\unskip\hfill\fi\oprocendsymbol}

\newcommand\fbmargin[1]{\marginpar[]{\color{blue} \tiny\ttfamily
		From FB: #1}}
\newcommand\sabermargin[1]{\marginpar[]{\color{blue}\tiny\ttfamily
		From SJ: #1}}

\usepackage{enumitem} \renewcommand{\theenumi}{(\roman{enumi})}
\renewcommand{\labelenumi}{\theenumi}

\DeclareSymbolFont{bbold}{U}{bbold}{m}{n}
\DeclareSymbolFontAlphabet{\mathbbold}{bbold}
\newcommand{\vect}[1]{\mathbbold{#1}}
\newcommand{\vectorones}[1][]{\vect{1}_{#1}}

\newcommand{\real}{\mathbb{R}}

\newcommand{\diagL}{\operatorname{diagL}}
\newcommand{\KM}{Krasnosel’skii–Mann\xspace}

\newcommand{\WP}[2]{\left\llbracket{#1}, {#2}\right\rrbracket}

\newcommand{\seminorm}[1]{{\left\vert\kern-0.25ex\left\vert\kern-0.25ex\left\vert #1
		\right\vert\kern-0.25ex\right\vert\kern-0.25ex\right\vert}}

\newcommand{\semimeasure}[1]{\mu_{\seminorm{\cdot}}\kern-0.5ex\left(#1\right)}

\newcommand{\osL}{\operatorname{osL}}
\newcommand{\Lip}{\operatorname{Lip}}

\newcommand{\norm}[2]{\|#1\|_{#2}}

\DeclareMathOperator{\diag}{diag}

\renewcommand{\top}{\mathsf{T}} %or \top or \intercal

\newcommand{\OF}{\mathsf{F}}
\newcommand{\OG}{\mathsf{G}}
\newcommand{\OH}{\mathsf{H}}
\newcommand{\OI}{\mathsf{I}}
\newcommand{\ON}{\mathsf{N}}

\usepackage{tikz}
\usepackage{pgfplots}
\usetikzlibrary{decorations.markings}
\usetikzlibrary{patterns}
\usetikzlibrary{arrows}
\usetikzlibrary{shapes}
\usetikzlibrary{fit}
 \usetikzlibrary{calc}
\usetikzlibrary{decorations.pathreplacing}
\usetikzlibrary{arrows,positioning} 
\usetikzlibrary{pgfplots.groupplots}
\usetikzlibrary{backgrounds}

\newcommand{\suchthat}{\;\ifnum\currentgrouptype=16 \middle\fi|\;}
\newcommand{\scirc}{\raise1pt\hbox{$\,\scriptstyle\circ\,$}}

\title[Robustness Certificates for Implicit Neural Networks]{Robustness
  Certificates for Implicit Neural Networks: \\ A Mixed Monotone Contractive Approach}

\author{%
 \Name{Saber Jafarpour}\thanks{These authors contributed equally} \Email{saber@gatech.edu}\\
 \addr Georgia Institute of Technology
 \vspace{-0.0125cm}
 \AND 
\Name{Matthew Abate}$^*$ \Email{matt.abate@gatech.edu}\\
 \addr Georgia Institute of Technology
 \vspace{-0.0125cm}
 \AND
 \Name{Alexander Davydov}$^*$ \Email{davydov@ucsb.edu}\\
 \addr University of California, Santa Barbara
 \vspace{-0.0125cm}
  \AND 
 \Name{Francesco Bullo} \Email{bullo@ucsb.edu}\\
 \addr University of California, Santa Barbara
 \vspace{-0.0125cm}
 \AND
 \Name{Samuel Coogan} \Email{sam.coogan@gatech.edu}\\
 \addr Georgia Institute of Technology}

\begin{document}

\maketitle

\begin{abstract}%
Implicit neural networks are a general class of learning models that replace the layers in traditional feedforward models with implicit algebraic equations.
Compared to traditional learning models, implicit networks offer competitive performance and reduced memory consumption. However, they can remain brittle with respect to input adversarial perturbations. 

This paper proposes a theoretical and computational framework for
robustness verification of implicit neural networks; our framework
blends together mixed monotone systems theory and contraction theory.
% In this paper, using mixed monotone system theory, we develop a framework 
% for robustness verification of implicit neural networks. 
%
First, given an implicit neural network, we introduce a related embedded network and show that, given an $\ell_\infty$-norm box constraint on the input, the embedded network provides an $\ell_\infty$-norm box overapproximation for the output of the given network.
Second, using $\ell_{\infty}$-matrix measures, we propose sufficient conditions for well-posedness of both the original and embedded system and design an iterative algorithm to compute the $\ell_{\infty}$-norm box robustness margins for reachability and classification problems.
Third, of independent value, we propose a novel relative classifier variable that leads to tighter bounds on the certified adversarial robustness in classification problems.
% Additionally, we study the notion of certified adversarial robustness and provide non-trivial lower bounds for certified adversarial robustness using our approach as well as existing verification approaches in the literature.
%
Finally, we perform numerical simulations on a Non-Euclidean Monotone Operator Network (NEMON) trained on the MNIST dataset. In these simulations, we compare the accuracy and run time of our mixed monotone contractive approach with the existing robustness verification approaches in the literature for estimating the certified adversarial robustness. 
\end{abstract}


\section{Introduction}

 Neural networks are increasingly being deployed in real-world applications, including natural language processing, computer vision, and self-driving vehicles. However, they are notoriously vulnerable to adversarial attacks; slight perturbations in the input can lead to large deviations in the output~\citep{CZ-WZ-IS-JB-DE-IG-RF:13}. Understanding this input sensitivity is essential in safety-critical applications, since the consequences of adversarial perturbations can be disastrous. Several different strategies have been proposed in the literature to design neural networks that are robust with respect to adversarial perturbations~\citep{IJG-JS-CZ:15,NP-PM-XW-SJ-AS:16}. Unfortunately, many of these approaches are based on robustness with respect to specific attacks and they do not provide formal robustness guarantees~\citep{AM-AM-LS-DT-AV:17,NC-DW:17}.
% \begin{wrapfigure}{r}{.5\textwidth}
%   \centering
%     \includegraphics[width=0.9\linewidth]{fig/adversarialPerturbations}
%   \caption{An adversarial example in handwritten digit classification: small input perturbations can lead to output misclassification}
%  \end{wrapfigure} 
 Recently, there has been a large interest in providing provable robustness guarantees for neural networks. Most existing approaches focus on either the $\ell_2$-norm or $\ell_{\infty}$-norm robustness measures. For neural networks with high-dimensional inputs and subject to dense perturbations, the $\ell_2$-norm robustness measures are known to provide overly conservative estimates of robustness and are less informative than their $\ell_{\infty}$-norm counterparts.  
 Rigorous verification methods generally fall into four different categories (i) Lipschitz bound methods~\citep{MF-AR-HH-MM-GJP:19,AV-KS:18,LPC-JCP:20}, (ii) interval bound methods~\citep{MM-TG-MV:18,SG-etal:18,HZ-etal:20}, (iii) optimization-based methods~\citep{EW-ZK:18,HZ-etal:18}, and (iv) probabilistic methods~\citep{JC-ER-JZK:19,BL-CC-WW-LC:19}. However, these methods suffer from several limitations. Regarding the Lipschitz bound approach, the proposed methods are either too conservative~\citep{CZ-WZ-IS-JB-DE-IG-RF:13} or not scalable to large-scale problems~\citep{AV-KS:18,LPC-JCP:20}. Similar concerns apply to interval-bound propagation methods and optimization-based methods.  Finally, probabilistic approaches provide some guarantees for $\ell_1$ and $\ell_2$-norm robustness but there are theoretical limitations in their applicability for certifying $\ell_\infty$-robustness~\citep{AB-TD-NM-HZ:20}.

In this paper we study the robustness properties of implicit neural networks, a recently proposed class of learning models with strong scalability properties.
Implicit neural networks replace the notion of layer from traditional neural networks with an implicit fixed-point equation~\citep{SB-JZK-VK:19,LEG-FG-BT-AA-AYT:21}. 
They can be considered as infinite-depth weight-tied neural networks where recursive function evaluation is performed via solving a single implicit algebraic equation. 
The implicit framework generalizes many classical neural networks including feedforward, convolutional, and residual networks~\citep{LEG-FG-BT-AA-AYT:21}. 
Implicit neural networks are inspired by biological systems and, compared to traditional neural networks, they offer competitive accuracy and reduced memory consumption~\citep{SB-JZK-VK:19}. 
Additionally, preliminary empirical evidence indicates that appropriately-trained implicit neural networks are more robust than traditional feedforward models~\citep{CP-EW-JZK:21}; however this phenomenon is not yet well understood and open questions remain regarding the stability and robustness of implicit models. 

We propose a rigorous computationally efficient certification method for implicit neural network robustness. We note that many of the classical robustness analysis tools for traditional neural networks are either not applicable to implicit neural networks or will lead to conservative results. Our novel approach is derived from mixed monotone systems theory and contraction theory. Unlike the robustness verification approaches based on estimates of Lipschitz constants, our framework takes into account how the $\ell_\infty$-error bounds propagate through the network and is scalable with the size of the network.

\subsection*{Related works}

\paragraph*{Implicit learning models.}

Implicit neural networks have been proposed as a generalization of feedforward neural networks~\citep{SB-JZK-VK:19,LEG-FG-BT-AA-AYT:21}.
%Deep equilibrium networks~\citep{SB-JZK-VK:19}. Implicit deep learning model~\citep{LEG-FG-BT-AA-AYT:21}. 
In~\citep{AK-ZZ-VS:20}, it is demonstrated that implicit models generally do not suffer from vanishing nor exploding gradients. One of the main challenges in studying implicit neural networks is their well-posedness, \emph{i.e.}, existence and uniqueness of solutions for their fixed-point equation. \citep{LEG-FG-BT-AA-AYT:21} proposes a sufficient spectral condition for convergence of the Picard iterations associated with the fixed-point
equation. In~\citep{EW-JZK:20,MR-RW-IRM:20}, using monotone operator theory, a suitable parametrization of the weight matrix is proposed
which guarantees the stable convergence of suitable fixed-point iterations.
Our previous work \citep{SJ-AD-AVP-FB:21f} proposes non-Euclidean contraction theory to design implicit neural networks and study their well-posedness, stability, and robustness with respect to the $\ell_\infty$-norm; the general theory is developed in~\citep{AD-SJ-FB:20o} and a short tutorial is given in~\citep{FB-PCV-AD-SJ:21e}.
% A recent survey on fixed point strategies in data science is given by~\citep{PLC-JCP:21}.

\paragraph*{Robustness of neural networks.}
Starting with~\citep{CZ-WZ-IS-JB-DE-IG-RF:13}, there has been a large body of work in machine learning to understand adversarial examples~\citep{AA-LE-AI-KK:18}. 
%
% Several different strategies have been proposed in the literature for designing neural networks that are robust with respect to adversarial perturbations~\citep{IJG-JS-CZ:15,NP-PM-XW-SJ-AS:16}. Unfortunately, many of these approaches are based on robustness with respect to specific attacks and they do not provide any formal robustness guarantees~\citep{AM-AM-LS-DT-AV:17,NC-DW:17}.
%
%Motivated by application of neural networks in safety-critical system, there has been a recent interest in providing robustness guarantees for them. 
Several examples for certified robustness training and analysis include~\citep{EW-ZK:18,HZ-etal:18,SG-etal:18,HZ-etal:20,MM-TG-MV:18,JC-ER-JZK:19}.
%In~\citep{EW-ZK:18} provable robustness via overapproximation of adversarial polytopes.  Linear relaxation-based robust verification (Crown) is proposed in~\citep{HZ-etal:18}. Interval bound Propagation (IBP) approach is first introduced in~\citep{SG-etal:18}. Combination of both (CROWN-IBP)~\citep{HZ-etal:20}. Mixture of zonotopic regions and
%interval bounds is studied in~\citep{MM-TG-MV:18}. 
% Improving robustness
% verification for classifiers with ReLU activation function based on
% Mixed-integer programming~\citep{KYX-VT-NMS-AM:19}\citep{JB-PV-KCW-RG-JHJ:21}
% \citep{BZ-TC-ZL-DH-LW:21}.
Regarding implicit neural networks, there are far fewer works on their robustness guarantees. In~\citep{LEG-FG-BT-AA-AYT:21} a sensitivity-based
robustness analysis for implicit neural network is proposed. Approximation of the Lipschitz constants of deep equilibrium
networks has been studied in~\citep{CP-EW-JZK:21,MR-RW-IRM:20}. Recently, the ellipsoid methods based on semi-definite programming~\citep{TC-JBL-VM-EP:21} and the interval-bound propagation method~\citep{anonymous:22} have been proposed for robustness
certification of deep equilibrium networks. 
%Another certified robustness approach is randomized smoothing~\citep{JC-ER-JZK:19}, which can provide robustness certificates for generic architectures including implicit neural networks. However, randomized smoothing has limitations in certifying $\ell_\infty$ robustness for high-dimensional inputs~\citep{AB-TD-NM-HZ:20}.

\paragraph*{Mixed monotone system theory.} Mixed monotone systems theory \citep{GE-HS-ES:06, DA-GE-ES:14, SC-MA:15b,SC:20} provides a generalization of classical monotone systems theory \citep{HLS:95,LF-SR:00,DA-EDS:03}, applicable to all dynamical systems bearing a locally Lipschitz continuous vector field \citep{LY-NO:19,MA-MD-SC:21}.  A dynamical system is mixed monotone when there exists a related decomposition function that separates the system's vector field or update map into increasing and decreasing components.  Such a decomposition then facilitates robustness analysis for the initial mixed monotone system and specifically enables, \emph{e.g.}, the efficient computation of robust reachable sets and invariant sets \citep{MA-SC:20}.

%Monotone system theory~\citep{HLS:95}, positive linear system~\citep{LF-SR:00}, Monotone systems with control input~\citep{DA-EDS:03}, mixed monotone dynamical system~\citep{HLS:06,HLS:08}. Mixed monotone~\citep{SC-MA:15b,LY-OM-NO:19,SC:20}, the notion of decomposition function~\citep{LY-NO:19,MA-MD-SC:21}.

\subsection*{Contributions}

Based on mixed monotone system theory, this paper proposes 
a theoretical and computational framework to study the
robustness of implicit neural networks. 
Given an implicit neural network, we introduce an associated embedded network with twice as many inputs and outputs as the original system. This embedded implicit network takes an $\ell_{\infty}$-norm box as its input and generates an $\ell_{\infty}$-norm box as its output.  
Then, we study the connection between the well-posedness of the embedded network and the robustness of the original implicit network. Our main theoretical contribution
is as follows: if the $\ell_{\infty}$-matrix measure of the original network's weight matrix is less than one, then (i) the implicit neural network has a unique fixed-point, (ii) the embedded network has a unique fixed-point which can be computed using a suitable average-iteration, and (iii) for a given $\ell_{\infty}$-norm box constraint on the input of the implicit neural network, the output of the embedded implicit neural network is an $\ell_{\infty}$-norm box overapproximation of output the original implicit network.
In particular, result (iii) above shows how bounds on the network output are obtained directly from bounds on the network input, allowing for efficient reachability analysis for implicit neural networks. However, the output bounds obtained using this approach can lead to conservative robustness estimates in classifications. 
As a practical contribution, we propose  a new classifier variable, again based
upon mixed monotone theory, that leads to sharper robustness estimates in classification. 
%
% In particular, result (iii) above shows how bounds on the network output are obtained directly from bounds on the network input, allowing for efficient reachability analysis for implicit neural networks.  
In order to evaluate the robustness guarantees of implicit neural networks, we 
empirically examine their certified adversarial robustness. We then use (i) estimates of Lipschitz bounds, (ii) the interval bound propagation method, and (iii) our mixed monotone contractive approach to provide lower bounds on certified adversarial robustness. 
Finally, we compare the certified adversarial robustness of the three approaches mentioned above on a pre-trained implicit neural network. Our simulation results show that the mixed monotone contractive approach significantly outperforms the other two methods.

\section{Mathematical preliminaries}

\paragraph*{Vectors and matrices.} 
%The set of $n\times n$ positive-definite
%matrices is denoted by $\mathbb{S}^{n}$. 
Given a matrix
$B \in \mathbb{R}^{n\times m}$, we denote the non-negative part of $B$
by $[B]^+ = \max(B, 0)$ and the nonpositive part of $B$ by
$[B]^- = \min(B, 0)$. The \emph{Metzler part} and the \emph{non-Metzler part} of square matrix $A\in \real^{n\times n}$
are denoted by $\lceil A \rceil^{\mathrm{Mzl}}\in \real^{n\times n}$ and
$\lfloor A \rfloor^{\mathrm{Mzl}}\in \real^{n\times n}$, respectively, where
\begin{align*}
  (\lceil A \rceil^{\mathrm{Mzl}})_{ij} &=\begin{cases}
    A_{ij} & A_{ij} \geq 0\; \mbox{or} \; i =  j\\
    0 & \mbox{otherwise,}
  \end{cases}\qquad \lfloor A\rfloor^{\mathrm{Mzl}}= A-\lceil A
        \rceil^{\mathrm{Mzl}}.
\end{align*}
We note that, for every square matrix $A\in \real^{n\times n}$, $\lceil A \rceil^{\mathrm{Mzl}}$ is a Metzler matrix and $\lfloor A\rfloor^{\mathrm{Mzl}}$ is a non-positive matrix with zero diagonal elements. For matrices $C\in \real^{n\times m}$ and $D\in \real^{p\times q}$, the Kronecker product of $C$ and $D$ is denoted by $C\otimes D$. 
%For every two vectors $v,w\in \real^n$ and every
%$i\in \{1,\ldots,n\}$, the vector obtained by replacing the
%$i^{\rm th}$ element of $v$ by the $i^{\rm th}$ element of $w$ is denoted by $v_{i:w}$.

\paragraph*{Matrix measures and weak pairings.}
For every $\eta\in \real^n_{>0}$, we define the diagonal matrix $[\eta]\in \real^{n\times n}$ by $[\eta]_{ii}=\eta_i$, for every $i\in \{1,\ldots,n\}$. For $\eta\in \real^n_{>0}$, the diagonally weighted
$\ell_{\infty}$-norm is defined by
$\|x\|_{\infty,[\eta]^{-1}}=\max_i|x_i|/\eta_i$, the diagonally weighted
$\ell_{\infty}$-matrix measure is defined by $\mu_{\infty,[\eta]^{-1}}(A) =\max_{i \in \until{n}} A_{ii} + \sum_{j \neq i} \frac{\eta_j}{\eta_i} |A_{ij}|$. We note that, for every $\eta\in \real^n_{>0}$ and every square matrix $A\in \real^{n\times n}$, we have
\begin{align}\label{eq:identity}
    \mu_{\infty,[\eta]^{-1}}(A) = \mu_{\infty,I_2\otimes [\eta]^{-1}}\left(\begin{bmatrix}\lceil A \rceil^{\mathrm{Mzl}} & \lfloor A \rfloor^{\mathrm{Mzl}}\\ \lfloor A \rfloor^{\mathrm{Mzl}}& \lceil A \rceil^{\mathrm{Mzl}}\end{bmatrix}\right).
\end{align}

From~\citep[Table~III]{AD-SJ-FB:20o}, we define the weak pairing $\map{\WP{\cdot}{\cdot}}{\real^n\times\real^n}{\real}$ associated to the norm $\|\cdot\|_{\infty,[\eta]^{-1}}$ as follows:
\begin{align*}
    \WP{x}{y}= \max_{i \in I_{\infty}([\eta]^{-1}y)} \eta_i^{-2} y_ix_i,
\end{align*}
where $I_{\infty}(x) = \setdef{i\in\until{n}}{|x_i|=\norm{x}{\infty}}$.

\paragraph*{Lipschitz and one-sided Lipschitz constants.}
Let $\OF:\real^n\times \real^m\to \real^n$ be a locally Lipschitz map in the first argument. For every $u\in \real^m$ and every $\alpha\in (0,1]$, we define the \emph{$\alpha$-average map} $\OF_{\alpha}:\real^n\times \real^m\to \real^n$ by $\OF_{\alpha} = (1-\alpha)\OI + \alpha \OF$, where $\OI$ is the identity map on $\real^n$. Given a positive vector $\eta\in \real^n_{>0}$, $\OF(x,u)$ is Lipschitz in $x$ with
respect to the norm $\|\cdot\|_{\infty,[\eta]^{-1}}$ with constant
$\Lip^x _{\infty,[\eta]^{-1}} (\OF)\in \real_{\ge 0}$ if, for every
$x_1,x_2\in \real^n$ and every $u\in \real^m$,
\begin{align*}
  \|\OF(x_1,u)-\OF(x_2,u)\|_{\infty,[\eta]^{-1}}\le \Lip^x _{\infty,[\eta]^{-1}} (\OF)\|x_1-x_2\|_{\infty,[\eta]^{-1}},
\end{align*}
For every $u\in \real^m$, we define the set $\Omega_u =\setdef{x\in \real^n}{\frac{\partial \OF(x,u)}{\partial x} \;\; \mbox{exists}}$. By Rademacher's theorem, the set $\real^n/\Omega_u$ is a measure zero set, for every $u\in \real^m$. The map $\OF(x,u)$ is one-sided Lipschitz in $x$ with respect to the norm $\|\cdot\|_{\infty,[\eta]^{-1}}$ with constant $\osL^x _{\infty,[\eta]^{-1}} (\OF)\in \real$ if, for every
$x_1,x_2\in \real^n$ and every $u\in \real^m$,
\begin{align*}
\WSIP{\OF(x_1,u)-\OF(x_2,u)}{x_1-x_2} \le  \osL^x_{\infty,[\eta]^{-1}}(\OF)\|x_1-x_2\|^2_{\infty,[\eta]^{-1}},
\end{align*}
and we define $\diagL(\OF)\in [-\Lip^x_{\infty,[\eta]^{-1}}(\OF),\Lip^x_{\infty,[\eta]^{-1}}(\OF)]$ by
\begin{align*}
    \diagL(\OF) = \min_{i\in \{1,\ldots,n\}} \inf_{u\in \real^m}\inf_{x\in \Omega_u} D_x\OF_{ii}(x,u),
\end{align*}

\paragraph*{Mixed monotone mappings.}
% We define the Southeast (SE) order $\le_{\mathrm{SE}}$ on
% $\real^{2n}\simeq\real^n\times \real^n$ by $\begin{bmatrix}
%   x\\
%   \widehat{x}
% \end{bmatrix} \le_{\mathrm{SE}}\begin{bmatrix}
%   y\\
%   \widehat{y} \end{bmatrix}$ if and only if $x\le y$ and
% $\widehat{x}\ge \widehat{y}$. 
Given a map $\OF:\real^n\times \real^m\to \real^n$ and a Lipschitz  function $d:\real^{2n}\times\real^{2m}\to \real^{n}$, we say $\OF$ is \emph{mixed monotone with respect to the decomposition function $d$}, if for every $i\in\{1,\ldots,n\}$,
\begin{enumerate}
    \item\label{p1:decom} $d_i(x,x,u,u)=\OF_i(x,u)$, for every $x\in \real^n$ and every $u\in \real^m$;
    \item\label{p2:decom} $d_i(x,\widehat{x},u,\widehat{u})\le
      d_i(y,\widehat{y},u,\widehat{u})$, for every $x\le y$ such
      that $x_i=y_i$, every $\widehat{y}\le \widehat{x}$, and every $u,\widehat{u}\in \real^m$;
    \item\label{p3:decom} $d_i(x,\widehat{x},u,\widehat{u})\le d_i(x,\widehat{x},v,\widehat{v})$, for every $u\le v$ , every $\widehat{v}\le \widehat{u}$, and every $x,\widehat{x}\in \real^n$.
    \end{enumerate}  
    \smallskip
Conditions \ref{p1:decom}--\ref{p3:decom} are sometimes referred to as the Kamke conditions for mixed monotonicity\footnote{These are the conditions for ensuring that the continuous-time dynamical system with vector field defined by such a mapping (possibly added to a scaling of identity) is mixed monotone.} as developed in \citep{MA-MD-SC:21}; see also \citep{SC:20} for a equivalent infinitesimal characterization of mixed monotonicity. Suppose that the map $\OF$ is linear, i.e., there exists $A\in \real^{n\times n}$ and $B\in \real^{n\times m}$ such that $\OF(x,u)=Ax+Bu$, for every $x\in \real^n$ and every $u\in \real^m$. Then one can easily show that $\OF$ is mixed monotone with respect to the following decomposition function~\cite[Example 3]{SC:20}:
\begin{align*}
    d(x,\widehat{x},u,\widehat{u}) = \lceil A\rceil^{\mathrm{Mzl}}x + \lfloor A\rfloor^{\mathrm{Mzl}} \widehat{x} + [B]^+u + [B]^-\widehat{u},\qquad \mbox{for every    }\;\; x,\widehat{x}\in \real^n,\;\; u,\widehat{u}\in \real^m.
\end{align*}
Indeed, one can show that every locally Lipschitz map $\OF$ is mixed monotone with respect to some decomposition function \citep{MA-MD-SC:21}, however, finding a closed form decomposition function is in general challenging. A remarkable property of implicit neural networks, shown below, is that an optimal decomposition function is easily available in closed-form.

\section{Implicit neural networks}

    An implicit neural network is described by the following fixed-point equation: 
    \begin{align}\label{eq:INN}
      x&=\Phi(Ax+Bu+b):=\ON(x,u),\nonumber\\
      y &= Cx+c,
    \end{align}
    where $x\in \real^n$ is the hidden variable, $u\in \real^{r}$ is
    the input and $y\in \real^{q}$ is the output. The matrices
    $A\in \real^{n\times n}$, $B\in \real^{n\times r}$, and
    $C\in \real^{q\times n}$ are weight
    matrices, $b\in \real^n$ and $c\in \real^q$ are bias vectors, and $\Phi(x)=(\phi_1(x_1),\ldots,\phi_n(x_n))^{\top}$ is the diagonal matrix of activation functions, where, for every $i\in \{1,\ldots,n\}$,
    $\phi_i:\real\to \real$ is weakly increasing and 
    satisfies $0\le \frac{\phi_i(x)-\phi_i(y)}{x-y}\le 1$, for every
    $x,y\in \real$. Compared to feedforward neural networks, one of
    the main challenges in studying implicit neural networks is their
    well-posedness; a unique solution for the fixed-point
    equation~\eqref{eq:INN} might not exist. We refer the readers to~\citep{EW-JZK:20,
      LEG-FG-BT-AA-AYT:21,MR-RW-IRM:20,SJ-AD-AVP-FB:21f} for discussions on the well-posedness of implicit networks. 
    
    \paragraph*{Training implicit neural networks}
  Given an input data $U=[u_1,\ldots,u_m]\in \real^{r\times m}$ and
  its corresponding output data
  $Y=[y_1,\ldots,y_m]\in \real^{q\times m}$, the goal of the training optimization
  problem is to learn weights and biases which minimizes
  $\mathcal{L}(Y,CX+c)$ subject to the fixed-point equation
  $X=\Phi(AX+BU)$, where
  $\mathcal{L}:\real^{q\times m}\times \real^{q\times m}\to \real$ is
  a suitable cost function. Thus, the training optimization problem is given by
  \begin{equation} \label{eq:TrainingProblem}
    \begin{aligned}
      \min_{A,B,C,b,c,X}\qquad
      &\mathcal{L}(Y, CX + c) \\
      & X = \Phi (AX+BU+b). 
    \end{aligned}
  \end{equation}
  In order to ensure that the implicit neural network is well-posed, an extra constraint is usually added to this training optimization problem. For instance, in~\citep{EW-JZK:20} the constraint $\mu_2(A)\leq\gamma$, in~\citep{LEG-FG-BT-AA-AYT:21} the constraint $\|A\|_{\infty}\le \gamma$, and in~\citep{SJ-AD-AVP-FB:21f} the constraint $\mu_{\infty,[\eta]^{-1}}(A)\le \gamma$ is proposed, for some $\gamma <1$ and some $\eta\in \real^n_{>0}$. 
%\sashatodo{Should we include the weighted parametrization here or in the experimental section?}
    
    \section{Robustness certificates via mixed monotone theory}
     One of the crucial features of neural networks in safety- and security-critical applications is their input-output robustness; the effect of input perturbations on the output. In this paper, we use the theory of mixed monotone systems to study robustness of implicit neural networks. 

    \paragraph*{Robustness of implicit neural networks.}
    We first introduce the embedded implicit neural network
     associated with~\eqref{eq:INN}. Given $\underline{u}\le\overline{u}$ in $\real^r$, we define \emph{embedded
      implicit neural network} by
      \begin{align}\label{eq:INN-embedding}
        \begin{bmatrix}\underline{x}\\\overline{x}\end{bmatrix}
        &= \begin{bmatrix}\Phi(\lceil A \rceil^{\mathrm{Mzl}} \underline{x}+\lfloor A \rfloor^{\mathrm{Mzl}} \overline{x} +
          [B]^{+}\underline{u} + [B]^{-}\overline{u} + b)\\ \Phi(\lceil A \rceil^{\mathrm{Mzl}}\overline{x}+\lfloor A \rfloor^{\mathrm{Mzl}}
          \underline{x} + [B]^{+}\overline{u} + [B]^{-}\underline{u}+b)\end{bmatrix}
        : = \begin{bmatrix}\ON^{\mathrm{E}}(\underline{x},
        \overline{x},\underline{u},\overline{u})\\ \ON^{\mathrm{E}}(\overline{x},\underline{x},
        \overline{u},\underline{u}) \end{bmatrix}, \nonumber\\
        % y_{[\underline{u},\overline{u}]} & = C [\underline{x},\overline{x}] + c.
        \begin{bmatrix}\underline{y}\\\overline{y}\end{bmatrix} &=
        \begin{bmatrix}[C]^+ & [C]^- \\ [C]^- & [C]^+\end{bmatrix}  \begin{bmatrix}\underline{x}\\\overline{x}\end{bmatrix} +
        \begin{bmatrix}c\\ c\end{bmatrix}. 
      \end{align}
      The embedded implicit neural network~\eqref{eq:INN-embedding} can be considered as a neural network with the box input $[\underline{u}, \overline{u}]$ and the box output $[\underline{y},\overline{y}]$ (see Figure~\ref{fig:INN}). 
      The following theorem studies well-posedness of the embedded
      implicit neural network~\eqref{eq:INN-embedding} and its
      connection with robustness of the implicit neural
      network~\eqref{eq:INN}.

      \begin{theorem}[Robustness of implicit neural
        networks]\label{thm:INN}
        Consider the implicit neural network~\eqref{eq:INN}. The following statement holds:
        \begin{enumerate}
        \item\label{p0:mixed} the map $\ON$ is mixed monotone with respect to the decomposition function $\ON^{\mathrm{E}}$;
        \end{enumerate}
        Moreover, let
        $\eta\in \real^n_{>0}$ be such that
        $\mu_{\infty,[\eta]^{-1}}(A)<1$. For every
        $\underline{u}\le \overline{u}$, every
        $u\in [\underline{u},\overline{u}]$, and every
        $\alpha\in [0,\alpha^*:=(1-\min_{i\in
          \{1,\ldots,n\}}(A_{ii})^{-})^{-1}]$,
        \begin{enumerate}\setcounter{enumi}{1}
        \item\label{p3:NE} the $\alpha$-average map
          $(\underline{x},\overline{x})\mapsto \begin{bmatrix}
          \ON^{\mathrm{E}}_{\alpha}(\underline{x},\overline{x},\underline{u},\overline{u})\\
          \ON^{\mathrm{E}}_{\alpha}(\overline{x},\underline{x},\overline{u},\underline{u})
          \end{bmatrix}$ is a contraction mapping with respect to the norm  $\|\cdot\|_{\infty,I_2\otimes [\eta]^{-1}}$ with minimum contraction factor
          $1-\frac{1-\mu_{\infty,[\eta]^{-1}}(A)^{+}}{1-\min_{i\in\until{n}}(A_{ii})^-}$;
        \item\label{p1:N} the $\alpha$-average map $\ON_{\alpha}$ is a
          contraction mapping with respect to the norm $\|\cdot\|_{\infty,[\eta]^{-1}}$ minimum contraction factor
          $\Lip(\ON_{\alpha^*}) =
          1-\frac{1-\mu_{\infty,[\eta]^{-1}}(A)^{+}}{1-\min_{i\in\until{n}}(A_{ii})^-}$;
        \item\label{p4:existenceNE} the embedded
          network~\eqref{eq:INN-embedding} has a unique fixed point
          $\begin{bmatrix}\underline{x}^*\\\overline{x}^*\end{bmatrix}$
          such that
          $\underline{x}^*\le \overline{x}^*$ and we have $\lim_{k\to
            \infty} \begin{bmatrix}\underline{x}^{k}\\\overline{x}^{k}\end{bmatrix}
          = \begin{bmatrix}\underline{x}^*\\\overline{x}^*\end{bmatrix}$,
          where the sequence
          $\left\{\begin{bmatrix}\underline{x}^{k}\\\overline{x}^{k}\end{bmatrix}\right\}_{k=1}^{\infty}$
          is defined iteratively by
          \begin{align}\label{eq:iterations-NE}
            \begin{bmatrix}\underline{x}^{k+1}\\\overline{x}^{k+1}\end{bmatrix} = \begin{bmatrix}\ON^{\mathrm{E}}_{\alpha^*}(\underline{x}^{k},\overline{x}^k,\underline{u},\overline{u})\\
            \ON^{\mathrm{E}}_{\alpha^*}(\overline{x}^k,\underline{x}^{k},\overline{u},\underline{u})
            \end{bmatrix}, \qquad\mbox{ for every }
            k\in \mathbb{Z}_{\ge
            0},\;\;  \begin{bmatrix}\underline{x}^{0}\\\overline{x}^{0}\end{bmatrix}\in
            \real^{2n};
          \end{align}
        \item\label{p2:existenceN} the implicit neural
          network~\eqref{eq:INN} has a unique fixed-point $x^*_u$ such
          that $x^*_u\in [\underline{x}^*,\overline{x}^*]$ and we have
          $\lim_{k\to \infty} x_u^k = x_u^*$ where the sequence
          $\{x_u^k\}_{k=1}^{\infty}$ is defined iteratively by
          \begin{align}\label{eq:iterations-N}
            x_u^{k+1}=\ON_{\alpha^*}(x_u^{k},u),\qquad\mbox{ for every }
            k\in \mathbb{Z}_{\ge 0}, \;\; x_u^0\in \real^n.
          \end{align}
        % \item\label{p5:improve} for every $\eta\in [\underline{x},\overline{x}]$, we have
        %   \begin{align*}
        %     \|\eta-x^*\|_{\infty,[\eta]^{-1}}\le \Lip_{\infty,[\eta]^{-1}}^{u\to x} \max\{\|u-\overline{u}\|_{\infty,[\eta]^{-1}},\|u-\underline{u}\|_{\infty,[\eta]^{-1}}\}.
        %   \end{align*}
        \end{enumerate}
      \end{theorem}
      \begin{proof}
      Regarding part~\ref{p0:mixed}, first note that we have $\ON(x,u) = \Phi(Ax+Bu)$, for every $x\in \real^n$ and every $u\in \real^r$. Moreover, for every $i\in \{1,\ldots,n\}$, the map $x\mapsto \phi_i(x)$ is globally Lipschitz and $(x,u)\mapsto Ax+Bu+b$ is an affine map. This implies that their composition map $(x,u)\mapsto \ON(x,u) = \Phi(Ax+Bu+b)$ is globally Lipschitz. Then one can use~\cite[Theorem 1]{MA-MD-SC:21} to construct a decomposition function for $\ON$, and thus the mapping $\ON$ is mixed monotone. Now, we show that $\ON^{\mathrm{E}}$ is a decomposition function for $\ON$. First note that, for every $x\in \real^n$ and $u\in \real^r$, we have
      \begin{align*}
          \ON^{\mathrm{E}}(x,x,u,u) = \Phi(\lceil A \rceil^{\mathrm{Mzl}}x + \lfloor A \rfloor^{\mathrm{Mzl}}x + [B]^{+}u + [B]^-u+b) = \Phi(Ax+Bu+b) = \ON(x,u).
      \end{align*}
      Moreover, pick $i\in \{1,\ldots,n\}$ and  $x,\widehat{x},y,\widehat{y}\in \real^n$ be such that $x\le y$ and $x_i=y_i$ and $\widehat{y}\le \widehat{x}$. It is easy to see that $(\lceil A \rceil^{\mathrm{Mzl}}x)_i \le (\lceil A \rceil^{\mathrm{Mzl}}y)_i$ and $(\lfloor A \rfloor^{\mathrm{Mzl}}\widehat{x})_i\le (\lfloor A \rfloor^{\mathrm{Mzl}}\widehat{y})_i$. As a result, for every $u,\widehat{u}\in \real^r$, we get
      \begin{align*}
          \ON_i^{\mathrm{E}}(x,\widehat{x},u,\widehat{u}) &= \phi_i\left((\lceil A \rceil^{\mathrm{Mzl}}x)_i + (\lfloor A \rfloor^{\mathrm{Mzl}}\widehat{x})_i + ([B]^{+}u)_i + ([B]^-\widehat{u})_i+b\right)\\ & \le \phi_i\left((\lceil A \rceil^{\mathrm{Mzl}}y)_i + (\lfloor A \rfloor^{\mathrm{Mzl}}\widehat{y})_i + ([B]^{+}u)_i + ([B]^-\widehat{u})_i+b\right) \\ & =  \ON_i^{\mathrm{E}}(y,\widehat{y},u,\widehat{u}),
      \end{align*}
      where the second inequality holds since $\phi_i$ is weakly increasing. Finally, for every $i\in \{1,\ldots,n\}$, let $u,\widehat{u},v,\widehat{v}\in \real^r$ be such that $u\le v$ and $\widehat{v}\le \widehat{u}$. It is easy to see that $([B]^{+}u)_i \le ([B]^{+}v)_i$ and $ ([B]^-\widehat{u})_i\le ([B]^{-}\widehat{v})_i$. As a result, for every $x,\widehat{x}\in \real^n$, we have
      \begin{align*}
          \ON_i^{\mathrm{E}}(x,\widehat{x},u,\widehat{u}) &= \phi_i\left((\lceil A \rceil^{\mathrm{Mzl}}x)_i + (\lfloor A \rfloor^{\mathrm{Mzl}}\widehat{x})_i + ([B]^{+}u)_i + ([B]^-\widehat{u})_i+b\right)\\ & \le \phi_i\left((\lceil A \rceil^{\mathrm{Mzl}}x)_i + (\lfloor A \rfloor^{\mathrm{Mzl}}\widehat{x})_i + ([B]^{+}v)_i + ([B]^-\widehat{v})_i+b\right) \\ & =  \ON_i^{\mathrm{E}}(x,\widehat{x},v,\widehat{v}), 
      \end{align*}
      where the second inequality holds since $\phi_i$ is weakly increasing. This shows that $\ON^{\mathrm{E}}$ is a decomposition function for the map $\ON$.
      
      Regarding part~\ref{p3:NE}, we define $\tilde{\Phi} = I_2\otimes \Phi$ and $\OG:\real^{2n}\to \real^{2n}$ by $\OG(\underline{x},\overline{x}) = \begin{pmatrix}\lceil A \rceil^{\mathrm{Mzl}} \underline{x} +\lfloor A \rfloor^{\mathrm{Mzl}}\overline{x}\\
          \lfloor A \rfloor^{\mathrm{Mzl}}\underline{x} +
          \lceil A
            \rceil^{\mathrm{Mzl}}\overline{x}\end{pmatrix}$. Additionally, we define $D = \begin{bmatrix}[B]^{+} & [B]^{-} \\ [B]^{-} & [B]^{+}\end{bmatrix}$ and $w=\begin{bmatrix} \underline{u}\\ \overline{u}
             \end{bmatrix}$. Then define $ \tilde{\Phi}^{\OG}: \real^{2n}\to \real^{2n}$ by
            \begin{align*}
             \tilde{\Phi}^{\OG}(\underline{x},\overline{x},w) :=   \begin{bmatrix}
           \ON^{\mathrm{E}}(\underline{x},\overline{x},\underline{u},\overline{u})\\
          \ON^{\mathrm{E}}(\overline{x},\underline{x},\overline{u},\underline{u})
          \end{bmatrix} = \tilde{\Phi}(\OG(\underline{x},\overline{x})+Dw+I_2\otimes b).
            \end{align*}
            The assumptions on each
        scalar activation function imply that (i)
        $\map{\tilde{\Phi}}{\real^{2n}}{\real^{2n}}$ is non-expansive with respect
        to $\norm{\cdot}{\infty,I_2\otimes [\eta]^{-1}}$ and (ii) for every
        $p,q \in \real$, there exists $\theta_i \in [0,1]$ such that
        $\phi_i(p) - \phi_i(q) = \theta_i(p - q)$ or in the matrix
        form $\tilde{\Phi}(\mathbf{p})-\Phi(\mathbf{q}) = \Theta (\mathbf{p}-\mathbf{q})$ where $\Theta\in \real^{2n\times 2n}$ is a
        diagonal matrix with diagonal elements $\theta_i\in
        [0,1]$ and $\mathbf{p},\mathbf{q}\in \real^{2n}$. As a result, for every $y_1,y_2\in \real^{2n}$, we have
        \begin{multline*}
          \|\tilde{\Phi}^{\OG}_{\alpha}(y_1,w)-\tilde{\Phi}^{\OG}_{\alpha}(y_2,w)\|_{\infty,I_2\otimes [\eta]^{-1}}
          = \|(1-\alpha)(y_1-y_2) + \alpha \Theta
            (\OG(y_1)-\OG(y_2))\|_{\infty,I_2\otimes [\eta]^{-1}} \\  \le \sup_{y\in \real^{2n}}\|I_{2n} + \alpha
                                                            (-I_{2n}
                                                            +\Theta D \OG(y))\|_{\infty,I_2\otimes [\eta]^{-1}}\|y_1-y_2\|_{\infty,I_2\otimes [\eta]^{-1}}. 
        \end{multline*}
        where the inequality holds by the mean value theorem. Then, for
        every $\alpha \in {]0,\frac{1}{1-\diagL(\Theta D\OG)}]}$,
        \begin{align*}
          \|I_{2n} + \alpha(-I_{2n}+\Theta D\OG(y))\|_{\infty,I_2\otimes [\eta]^{-1}} &= 1 +
          \alpha \mu_{\infty,I_2\otimes[\eta]^{-1}}\big(-I_{2n}+\Theta D\OG(y)\big) \\
          &
          =
          1
          +
          \alpha
          \big(-1
          +
          \mu_{\infty,I_2\otimes[\eta]^{-1}}(\Theta D\OG(y))\big)\\ &
          \le
          1
          +
          \alpha \big(-1+\mu_{\infty,I_2\otimes [\eta]^{-1}}(D\OG(y))^{+}\big) \\ & \le 1 -
          \alpha (1-\mu_{\infty,[\eta]^{-1}}(A)^{+}) < 1, 
        \end{align*}
        where the first equality holds by~\cite[Lemma 7(i)]{SJ-AD-AVP-FB:21f}, the second
        equality holds by translation property of matrix
        measures, the third
        inequality holds by~\cite[Lemma 8(i)]{SJ-AD-AVP-FB:21f}, and the fourth inequality holds by~\eqref{eq:identity}. Moreover, since $\theta_i\in
        [0,1]$, we have $\theta_i (D\OG)_{ii} \ge (D\OG)_{ii}^-$, for every
        $i\in \{1,\ldots,2n\}$. This means that 
        \begin{align*}
          \diagL(\Theta D \OG) = \min_{i}\inf_{y\in \real^{2n}}
          (\Theta D \OG (y))_{ii} \ge  \min_{i}\inf_{y\in \real^{2n}}
          (D\OG_{ii}(y))^- = \min_{i\in \{1,\ldots,n\}}(A_{ii})^-. 
          \end{align*}
        This implies that,  for
        every $\alpha \in (0,(1-\min_{i\in\{1,\ldots,n\}}(A_{ii})^{-})^{-1}]$,
        \begin{align*}
          \|\tilde{\Phi}^{\OG}_{\alpha}(x_1,u)-\tilde{\Phi}^{\OG}_{\alpha}(x_2,u)\|_{\infty,I_2\otimes [\eta]^{-1}}\le
          (1 -\alpha (1-\mu_{\infty,[\eta]^{-1}}(A)^{+})) \|x_1-x_2\|_{\infty,I_2\otimes [\eta]^{-1}}. 
        \end{align*}
        Since $1 -\alpha (1-\mu_{\infty,[\eta]^{-1}}(A)^{+}) < 1$,
        $\tilde{\Phi}^{\OG}_{\alpha}(\cdot,w)$ is a contraction mapping with respect to the norm $\|\cdot\|_{\infty,I_2\otimes [\eta]^{-1}}$ for every
        $\alpha \in (0,(1-\min_{i\in\{1,\ldots,n\}}(A_{ii})^{-})^{-1}]$. 
      
      Regarding part~\ref{p1:N}, the proof follows by applying the same argument as in the proof of part~\ref{p3:NE} and using $\Phi(Ax+Bu+b)$ instead of $\tilde{\Phi}^{\OG}(\underline{x},\overline{x},\underline{u},\overline{u})$. 
      
      Regarding parts~\ref{p4:existenceNE} and~\ref{p2:existenceN}, by part~\ref{p3:NE}, the $\alpha$-average map $(\underline{x},\overline{x})\mapsto \begin{bmatrix}
          \ON^{\mathrm{E}}_{\alpha^*}(\underline{x},\overline{x},\underline{u},\overline{u})\\
          \ON^{\mathrm{E}}_{\alpha^*}(\overline{x},\underline{x},\overline{u},\underline{u})
          \end{bmatrix}$ is a contraction mapping  with respect to the norm $\|\cdot\|_{\infty,I_2\otimes [\eta]^{-1}}$. Therefore, by Banach's contraction mapping theorem, this map has a unique fixed-point  $\begin{bmatrix} \underline{x}^*\\ \overline{x}^*
       \end{bmatrix}$ and the iteration~\eqref{eq:iterations-NE} computes this fixed point. The fact that $\begin{bmatrix} \underline{x}^*\\ \overline{x}^*
       \end{bmatrix}$ is also the unique fixed-point of $(\underline{x},\overline{x})\mapsto \begin{bmatrix}
          \ON^{\mathrm{E}}(\underline{x},\overline{x},\underline{u},\overline{u})\\
          \ON^{\mathrm{E}}(\overline{x},\underline{x},\overline{u},\underline{u})
          \end{bmatrix}$ is a straightforward consequence of the following implications
        \begin{align*}
          \begin{bmatrix}
            \underline{x}^*\\
            \overline{x}^*\end{bmatrix}
          =
          \begin{bmatrix}
          \ON^{\mathrm{E}}_{\alpha^*}(\underline{x}^*,
          \overline{x}^*,\underline{u},\overline{u})\\ 
            \ON^{\mathrm{E}}_{\alpha^*}(\underline{x}^*,
          \overline{x}^*,\underline{u},\overline{u}) \end{bmatrix} &\iff \begin{bmatrix}
            \underline{x}^*\\
            \overline{x}^*\end{bmatrix} =(1-\alpha^*)\begin{bmatrix}
            \underline{x}^*\\
            \overline{x}^*\end{bmatrix} + \alpha^* \begin{bmatrix}
          \ON^{\mathrm{E}}(\underline{x}^*,
          \overline{x}^*,\underline{u},\overline{u})\\ 
            \ON^{\mathrm{E}}(\underline{x}^*,
          \overline{x}^*,\underline{u},\overline{u}) \end{bmatrix} \\ & \iff \begin{bmatrix}
            \underline{x}^*\\
            \overline{x}^*\end{bmatrix} =\begin{bmatrix}
          \ON^{\mathrm{E}}(\underline{x}^*,
          \overline{x}^*,\underline{u},\overline{u})\\ 
            \ON^{\mathrm{E}}(\underline{x}^*,
          \overline{x}^*,\underline{u},\overline{u}). \end{bmatrix}
        \end{align*}
        Similar argument can be used to prove existence and uniqueness of the fixed-point $x^*_u$ for $\ON$ and one can show iteration~\eqref{eq:iterations-N} converges to this fixed-point. Now, we show that $\underline{x}^*\le x^*_u\le \overline{x}^*$. We choose the initial condition $\begin{bmatrix}\underline{x}^0 \\ \overline{x}^0\end{bmatrix}$ for the iteration~\eqref{eq:iterations-NE} and choose an initial condition $x^{0}_u\in \real^n$ satisfying $\underline{x}^0\le x^{0}_u \le \overline{x}^0$ for the iteration~\eqref{eq:iterations-N}. We prove by induction that, for every $k\in \mathbb{Z}_{\ge 0}$, we have $\underline{x}^{k}\le x^{k}_u\le \overline{x}^{k}$. Suppose that this claim is true for $k\in \{1,\ldots,m\}$ and we show that this claim is true for $k=m+1$. Note that 
        \begin{align*}
         \underline{x}^{m+1}-x^{m+1}_u &=(1-\alpha^*)(\underline{x}^{m}-x^m_u) \\ & + \alpha^* (\Phi(\lceil A \rceil^{\mathrm{Mzl}} \underline{x}^m +\lfloor A \rfloor^{\mathrm{Mzl}}\overline{x}^m +[B]^{+}\underline{u}+[B]^-\overline{u}+b) - \Phi(A x^m_{u}+Bu + b))\\ &= \left((1-\alpha^*)I_n + \alpha^*\Theta \lceil A \rceil^{\mathrm{Mzl}}\right)(\underline{x}^{m}-x^m_u) + \alpha^*\Theta \lfloor A \rfloor^{\mathrm{Mzl}} (\overline{x}^m-x^m_u) \\ &+ \alpha^*\Theta[B]^{+}(\underline{u}-u) + \alpha^*\Theta[B]^-(\overline{u}-u), 
        \end{align*}
        where the non-negative diagonal matrix $\Theta =\diag(\theta_i)\in \real^n$ is defined as follows: for every $i\in \{1,\ldots,n\}$, $\theta_i\in [0,1]$ is such that $\phi_i(p_i)-\phi_i(q_i) = \theta_i(p_i-q_i)$, where $p=\lceil A \rceil^{\mathrm{Mzl}} \underline{x}^m +\lfloor A \rfloor^{\mathrm{Mzl}}\overline{x}^m + [B]^+\underline{u} + [B]^-\overline{u} +b $ and $q=A x^m_{u}+Bu + b$. Moreover, we know that $\Theta \lfloor A \rfloor^{\mathrm{Mzl}}\le \vect{0}_{n\times n}$ and, for every $i\in \{1,\ldots,n\}$, we have
        \begin{align*}
            (1-\alpha^*) + \alpha^*\theta_i A_{ii} \ge (1-\alpha^*) + \alpha^* A_{ii}^{-} \ge 0.
        \end{align*}
        This implies that $(1-\alpha^*)I_n + \alpha^*\Theta \lceil A \rceil^{\mathrm{Mzl}}\ge \vect{0}_{n\times n}$. Additionally, we have $\Theta [B]^{+}\ge \vect{0}_{n\times r}$ and $\Theta [B]^{-}\le \vect{0}_{n\times r}$. Therefore, using the induction assumption, we get $\underline{x}^{m+1}-x^{m+1}_u\le \vect{0}_n$. Similarly, one can show that $x^{m+1}_u-\overline{x}^{m+1}\le \vect{0}_n$. This completes the proof of induction. As a result, we get
        \begin{align*}
            \underline{x}^* = \lim_{k\to \infty}\underline{x}^{k} \le \lim_{k\to \infty} x^k_u =x^*_u \le \lim_{k\to \infty}\overline{x}^{k}= \overline{x}^*.
        \end{align*}
        This completes the proof of the theorem. 
                \end{proof}

 \begin{remark}\begin{enumerate}[nosep]
    \item Theorem~\ref{thm:INN} can be interpreted as a dynamical
      system approach to study robustness of implicit neural networks. Indeed, it is easy to see that the $\alpha$-average
      iteration~\eqref{eq:iterations-NE}
      (resp.~\eqref{eq:iterations-N}) are the forward Euler
      discretization of the dynamical system
      $\frac{d}{dt}\begin{bmatrix}\underline{x}\\ \overline{x}\end{bmatrix} =
      -\begin{bmatrix}\underline{x}\\ \overline{x}\end{bmatrix} +
      \begin{bmatrix}
          \ON^{\mathrm{E}}(\underline{x},\overline{x},\underline{u},\overline{u})\\
          \ON^{\mathrm{E}}(\overline{x},\underline{x},\overline{u},\underline{u})
          \end{bmatrix}$
      (resp. $\frac{dx}{dt}=-x+\ON(x,u)$).
      
      \item Theorem~\ref{thm:INN}\ref{p4:existenceNE} and \ref{p2:existenceN} show that $\mu_{\infty,[\eta]^{-1}}(A)<1$ is a sufficient condition for existence and uniqueness of the fixed-point of both the original neural network and embedded neural network. In~\citep{SJ-AD-AVP-FB:21f}, to ensure well-posedness,  the NEMON model is trained by adding the sufficient condition $\mu_{\infty,[\eta]^{-1}}(A)<1$  to the training problem~\eqref{eq:TrainingProblem}. Therefore, for the NEMON model introduced in~\citep{SJ-AD-AVP-FB:21f}, the embedded implicit network provides a margin of robustness for the original neural network with respect to any $\ell_\infty$-norm box uncertainty on the input.   
      
      \item In terms of evaluation time, computing the $\ell_{\infty}$-norm box bounds on the output is equivalent to two forward passes of the original implicit network (see Figure~\ref{fig:INN}).
      
      \item Implicit neural networks contain feedforward neural networks as a special case~\citep{LEG-FG-BT-AA-AYT:21}. Indeed, for a feedforward neural network with $k$ layers and $n$ neurons in each layer, there exists an implicit network representation with block upper diagonal weight matrix $A\in \real^{kn\times kn}$ and a vector $\eta\in \real^{kn}_{>0}$ such that $\mu_{\infty,[\eta]^{-1}}(A)<1$. In this case, the fixed-point of the embedded implicit network~\eqref{eq:INN-embedding} is unique, can be computed explicitly, and corresponds exactly to the approach taken in~\citep{SG-etal:18}.

        %   \item Theorem~\ref{thm:INN}\ref{p5:improve} guarantees that the bound
        %     obtained using the mixed mmonotone approach (i.e.,
        %     Theorem~\ref{thm:INN}\ref{p4:existenceNE}) is tighter than
        %     (or equal to) any bound obtained
        %     from Lipschitz constant of the neural network. 
  \end{enumerate}
\end{remark}

\begin{figure}
  \includegraphics[width=\linewidth]{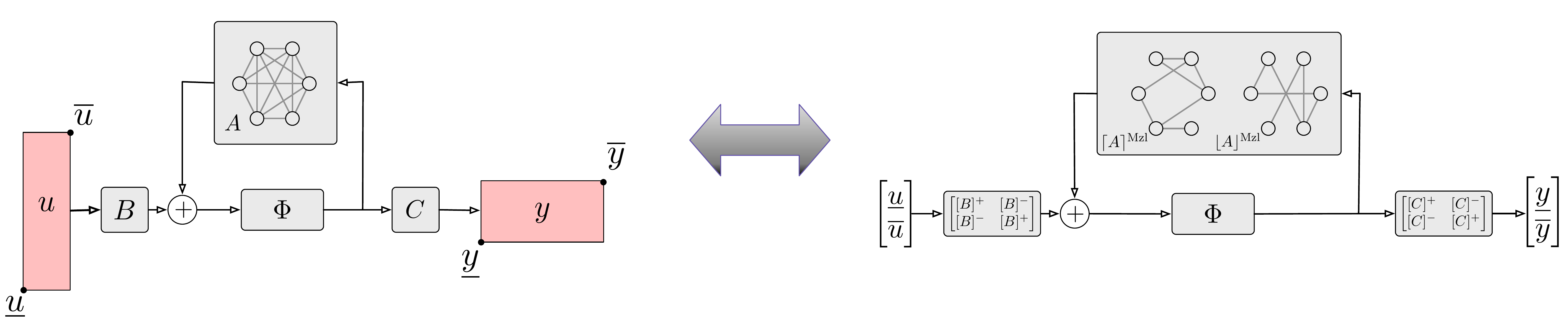}
  \vspace{-0.8cm}
  \caption{The original implicit neural network and its associated
    embedding network. The input-output behavior of
    the embedding system provides a box estimate for robustness of  the
    original  network.}\label{fig:INN}
\end{figure}

\paragraph*{Robustness verification via relative classifiers.}

The embedded network output $[\underline{y},\overline{y}]$ provides bounds on the elements of the initial implicit network's output, thus allowing for efficient  reachability analysis.  However, for classification problems, where the goal is to identify the maximum element of $y$, these boxes can lead to overly conservative estimates of robustness. In this section, we propose an
alternative approach to study classification problem by introducing a new classifier variable. Suppose the input $u\in\real^r$ leads to the output $y(u)\in\real^q$ and the correct label of  $u$ is $i \in \{ 1, \ldots, q\}$. %Given bounds on the control input $u \in [\underline{u}, \overline{u}]$, 
We are interested to study the robustness of our classifier with respect to a perturbed set of inputs $[\underline{u}, \overline{u}]\ni u$.  For every $v\in [\underline{u},\overline{u}]$, we propose the \emph{relative classifier variable} $z^u(v)\in \real^{q-1}$ defined by
\begin{equation}\label{eq_classifier}
   z^{u}(v) := y(v)_i \vect{1}_{q-1} - y(v)_{-i},
\end{equation}
where $y(v)_{-i} = (y(v)_1,\ldots, y(v)_{i-1},y(v)_{i+1},\ldots,y(v)_{q})^{\top}\in \real^{q-1}$. Note that $z^u(v) > 0$ only when the perturbed input $v$ retains the correct label $i$, i.e., the perturbation does not have any effect on the classification. Using~\eqref{eq:INN}, we write~\eqref{eq_classifier} as 
\begin{equation}\label{eq_zuv}
   z^u(v) = T^u y(v)  = T^u C x^* + T^u c,
\end{equation}
where $x^*$ is the fixed-point of the implicit neural network~\eqref{eq:INN} with input $v$ and $T^u\in\{-1,0,1\}^{(q-1)\times{q}}$ is the linear transformation defined by \eqref{eq_classifier}.  Now, we construct  
\begin{equation}\label{eq:underline-z}
    \underline{z}^u  = [T^u C]^+ \underline{x}^* + [T^u C]^- \overline{x}^* + T^u c,
\end{equation}
where $\underline{x}^*, \overline{x}^*$ solves \eqref{eq:INN-embedding} with $\underline{u}, \overline{u}$ being the above perturbation bounds on the input.  
\begin{lemma}[Properties of the relative classifier variable]\label{lem:conservative}
Let $u\in [\underline{u},\overline{u}]$ be an input with the correct label $i\in \{1,\ldots,q\}$ and $\begin{bmatrix}\underline{y}\\\overline{y}\end{bmatrix}$ be the output of the embedded network~\eqref{eq:INN-embedding} with input $\begin{bmatrix}\underline{u}\\\overline{u}\end{bmatrix}$. Then,
\begin{enumerate}
    \item\label{p1} $\underline{z}^u>0$ implies that the every perturbed input $v\in [\underline{u},\overline{u}]$ is given the same label as $u$, that is, $y_i(v) > y_j(v)$ for all $j \neq i$ and every $v\in [\underline{u},\overline{u}]$;
    \item\label{p2} $\underline{y}_i-\max_{j\ne
  i}\overline{y}_j> 0$ implies that $\underline{z}^u> 0$. 
\end{enumerate}
\end{lemma}
\begin{proof}  Choose $u$ with correct label $i$ and suppose $\underline{u} \leq \overline{u}$ so that $u \in [\underline{u}, \overline{u}]$. Additionally, let $\left[\begin{smallmatrix}\underline{x}^* \\ \overline{x}^*\end{smallmatrix}\right]$ and $\left[\begin{smallmatrix}\underline{y} \\ \overline{y}\end{smallmatrix}\right]$ be the state and output solutions to \eqref{eq:INN-embedding} for input $\left[\begin{smallmatrix}\underline{u} \\ \overline{u}\end{smallmatrix}\right]$.

  Regarding part~\ref{p1}, we observe that  $T^u C x \geq [T^u C]^+ \underline{x} + [T^u C]^- \overline{x}$ for any $x \in [\underline{x}, \overline{x}]$, and thus $z^u(v) \ge \underline{z}^u$ for any $v \in [\underline{u}, \overline{u}]$.  In this way, $\underline{z}^u > 0$ implies that $z^u(v) > 0$ for all $v \in [\underline{u}, \overline{u}]$, \emph{i.e.}, $v$ is given the label $i$ for all $v \in [\underline{u}, \overline{u}]$.
  
  Regarding part~\ref{p2}, suppose that $\underline{y}_i-\max_{j\ne
  i}\overline{y}_j> 0$. We note that $\underline{y}_i \vect{1}_{q-1} -\overline{y}_{-i} = [T^u]^+\underline{y} + [T^u]^-\overline{y}$ and therefore, we have  $[T^u]^+\underline{y} + [T^u]^-\overline{y}  \geq 0$. Now observe that
  \begin{equation}
    \begin{split}
        [T^u]^+\underline{y} + [T^u]^-\overline{y} &= [T^u]^+([C]^+\underline{x}^* + [C]^-\overline{x}^* + c) + [T^u]^-([C]^-\underline{x}^* + [C]^+\overline{x}^* + c) \\
        &= ([T^u]^+[C]^+ + [T^u]^-[C]^-) \underline{x}^* + ([T^u]^+[C]^- + [T^u]^-[C]^+)\overline{x}^* + T^u c. 
    \end{split}
  \end{equation}
  Now, note that 
  \begin{align*}
      T^u C & = ([T^u]^{+}+[T^u]^{-})([C]^{+}+[C]^{-}) \\ & = ([T^u]^+[C]^+ + [T^u]^-[C]^-) + ([T^u]^+[C]^- + [T^u]^-[C]^+),
  \end{align*}
  where $[T^u]^{+}[C]^{+}+[T^u]^{-}[C]^{-}\ge \vect{0}_{q-1\times n}$ and $[T^u]^+[C]^- + [T^u]^-[C]^+\le \vect{0}_{q-1\times n}$. On the other hand, we know that $T^u C = [T^u C]^{+} + [T^u C]^{-}$. This implies that 
  \begin{align*}
      [T^u C]^{+} &\le [T^u]^{+}[C]^{+}+[T^u]^{-}[C]^{-},\\
      [T^u C]^{-} &\ge [T^u]^+[C]^- + [T^u]^-[C]^+.
  \end{align*}
  Additionally, by Theorem~\ref{thm:INN}\ref{p4:existenceNE}, we know that $\underline{x}^*\le \overline{x}^*$. Thus, we get 
  \begin{align*}
       [T^u]^+\underline{y} + [T^u]^-\overline{y} & = ([T^u]^+[C]^+ + [T^u]^-[C]^-) \underline{x}^* + ([T^u]^+[C]^- + [T^u]^-[C]^+)\overline{x}^* + T^u c \\ & \le [T^u C]^+\underline{x}^* + [T^u C]^-\overline{x}^*+T^u c = \underline{z}^u.
  \end{align*}
  Thus, if $\underline{y}_i-\max_{j\ne
  i}\overline{y}_j> 0$ then $\underline{z}^u> 0$.  This completes the proof.
\end{proof}

Note that the converse of Lemma~\ref{lem:conservative}\ref{p2} need not hold in general. Indeed, Lemma~\ref{lem:conservative} shows that using $\underline{z}^u > 0$ for classification  leads to less conservative robustness certificates compared to using $\underline{y}_i-\max_{j\ne i}\overline{y}_j> 0$. 
%  {\color{red} We argue that checking $\underline{z}^u \geq 0$ leads to more accurate robustness certificates, rather than XXX}

\section{Theoretical and numerical comparisons}
In this section, we compare our robustness bounds with the existing
bounds in the literature. Before we proceed with the comparison, following~\citep{SG-etal:18,CP-EW-JZK:21}, we introduce the notion of
certified adversarial robustness which plays a crucial role in our
numerical comparison for classification problems. 
Given an implicit neural network~\eqref{eq:INN}, its 
certified adversarial robustness is its accuracy for
detection of the correct label. 
To this end, we consider a set of labeled test data $\mathcal{U}\subset \real^r$ and 
%assume that there exists a
%probability distribution $\mathcal{U}$ on the space of inputs with the
%probability function $\mathbb{P}$. 
we define the \emph{deviation function}
$\delta:\real_{\ge 0}\times \mathcal{U}\to \real$ by 
\begin{align}
  \delta(\epsilon,u) =\max_{v\in\real^r} \; \setdef{y(v)_i-\max_{j\ne
  i}y(v)_j}{\|u-v\|_{\infty}\le \epsilon,\;\; i \;\;\mbox{ is the correct label of }u},
  \end{align}
  where $y(u)$ and $y(v)$ are the implicit neural network 
  outputs generated by inputs $u$ and $v$ respectively.
  We say that the network is \emph{certified adversarially robust} for radius $\epsilon$ at input $u$ if $\delta(\epsilon,u)>0$. 
  
% Then, the \emph{certified adversarial robustness} $\mathrm{CAR}:\real_{\ge
%   0}\to [0,1]$ is defined by 
% \begin{align*}
%   \mathrm{CAR}(\epsilon) =
%   \mathbb{P}_{\mathcal{U}}[\delta(\epsilon,u)>0].
% \end{align*}
Certifying adversarial robustness can be complicated due to
the non-convexity of the optimization problem on $v$ for the deviation
function. We briefly review the existing
methods for robustness verification of implicit neural networks and show how these methods provide upper bounds on the deviation function and thus a lower bound on the certified adversarial
robustness.

\paragraph*{Method 1: Lipschitz constants.}
For implicit neural network, the estimates on
the input-output Lipschitz constants are studied for deep equilibrium
networks in~\citep{EW-JZK:20,CP-EW-JZK:21,MR-RW-IRM:20}, for implicit
deep learning models in~\citep{LEG-FG-BT-AA-AYT:21}, and for
non-Euclidean monotone operator networks in~\citep{SJ-AD-AVP-FB:21f}. For an implicit neural network~\eqref{eq:INN} with $\ell_\infty$ input-output Lipschitz constant
$\Lip^{u\to y}_{\infty}\in \real_{\ge 0}$, the output can be bounded
as $\|y(u)-y(v)\|_{\infty}\le \Lip^{u\to
  y}_{\infty}\|u-v\|_{\infty}$. We define
$\supscr{\delta}{Lip}(\epsilon,u) :=
(y(u)_i-\max_{j\ne i}y(u)_j)-2(\Lip^{\infty}_{u\to y})\epsilon$. One can see that $\supscr{\delta}{Lip}(\epsilon,u)>0$ is a sufficient condition for certified adversarial robustness. 

% Therefore, we can find the following
% upper bounds on the certified adversarial robustness
% $ \mathrm{CAR}(\epsilon) \le \mathrm{CAR}^{\Lip}(\epsilon)
% :=\mathbb{P}_{\mathcal{U}}[\supscr{\delta}{Lip}(\epsilon,u)> 0]$.
% \paragraph*{Ellipsoid method.} Recently, the ellipsoid methods based
% on semi-definite programming have proposed for robustness
% certification of implicit neural networks using
% ellipsoids~\citep{TC-JBL-VM-EP:21}. Given an implicit neural
% network~\eqref{eq:INN} with input perturbation
% $\|u-v\|_{\infty}\le \epsilon$, the output can be bounded using the
% ellipsoid $\|Qy^{v}+b\|_{2}\le 1$ where $Q$ and $b$ are solutions of
% the following optimization problem
% \begin{align}\label{eq:ellipse}
%   \min_{Q\in \mathbb{S}^{q}, b\in \real^{q}} \setdef{\det(Q)}{\|u-v\|_{\infty}\le
% \epsilon, \; \|Qy^{v}+b\|_{2}\le 1}
% \end{align}
% This optimization problem is then relaxed to a Sum-of-Square
% constrained problem and solved using a hierarchy of semi-definite
% programs~\citep{TC-JBL-VM-EP:21}. We first define
% $\supscr{\delta}{SDP}(\epsilon,u) =\max_{z\in \real^{q}}
% \setdef{z_i-\max_{j\ne i}z_i}{\|Qz+b\|_2\le 1, \;\;
% i=\mathrm{argmax}_{j}\; y^u_j}$, where $Q$ and $b$ are the solutions
% of the optimization problem~\eqref{eq:ellipse}. Then, one can find
% the following upper bound for certified adversarial robustness using
% ellipsoid method
% $ \mathrm{CAR}(\epsilon) \le \mathrm{CAR}^{\mathrm{SDP}}(\epsilon):=
% \mathbb{P}_{\mathcal{U}}[\supscr{\delta}{SDP}(\epsilon,u)> 0]$.

\paragraph*{Method 2: Interval bound propagation.} 
In~\citep{SG-etal:18} a framework based on interval
bound propagation has been proposed for training robust feedforward neural networks. This method has recently been extended for training deep equilibrium networks in~\citep{anonymous:22}. Given an implicit neural
network~\eqref{eq:INN} with input perturbation
$\|u-v\|_{\infty}\le \epsilon$, we can adopt the approach in~\citep{SG-etal:18} to the implicit framework and propose the following fixed-point equation for estimating the output of the network: 
\begin{align}
    \begin{bmatrix}
      \underline{x}\\
      \overline{x}
    \end{bmatrix} &= \begin{bmatrix}
      \Phi([A]^{+}\underline{x}+[A]^{-}\overline{x}+[B]^+\underline{u}
      + [B]^{-}\overline{u} + b)\\
      \Phi([A]^{+}\overline{x}+[A]^{-}\underline{x}+[B]^+\overline{u}
      + [B]^-\underline{u} + b)
    \end{bmatrix} , \label{eq:IBP-1}\\
    \begin{bmatrix}
      \underline{y}\\
      \overline{y}
    \end{bmatrix} &= \begin{bmatrix}
      [C]^{+} & [C]^{-}\\
      [C]^{-} & [C]^{+}
    \end{bmatrix}\begin{bmatrix}
      \underline{x}\\
      \overline{x}
    \end{bmatrix} + \begin{bmatrix}
      c\\
      c
    \end{bmatrix}\label{eq:IBP-2}  ,     
\end{align}
where $\underline{u}=u-\epsilon\vect{1}_m$,
$\overline{u}=u+\epsilon\vect{1}_m$, and $(\underline{x},\overline{x})$ are the solutions of the fixed-point equation~\eqref{eq:IBP-1}. It is worth mentioning that the condition $\mu_{\infty,[\eta]^{-1}}(A) < 1$ proposed in Theorem~\ref{thm:INN} does not, in general, ensure well-posedness of the fixed-point equation~\eqref{eq:IBP-1}. The output of the neural
network then can be bounded by the box
$[\underline{y},\overline{y}]$. We define $\supscr{\delta}{IBP}(\epsilon,u) =\underline{y}_i - \max_{j\ne
  i} \overline{y}_i$. One can see that $\supscr{\delta}{IBP}(\epsilon,u)>0$ is a sufficient condition for certified adversarial robustness.  

% \sashatodo{Mention that the sufficient condition $\mu_{\infty,[\eta]^{-1}}(A) < 1$ does not, in general, ensure convergence of these iterations}

\paragraph*{Method 3: Mixed monotone contractive approach.} Given an implicit neural
network~\eqref{eq:INN} with input perturbation
$\|u-v\|_{\infty}\le \epsilon$, we first use
Theorem~\ref{thm:INN} to obtain bounds on the output of the network. Indeed,
by Theorem~\ref{thm:INN}\ref{p3:NE}, the $\alpha$-average iteration~\eqref{eq:iterations-NE} with $\underline{u}=u-\epsilon\vect{1}_m$,
$\overline{u}=u+\epsilon\vect{1}_m$ converges to
$(\underline{x},\overline{x})$ and therefore, we have $y(v)\in
[\underline{y},\overline{y}]$. Moreover, we can define $\supscr{\delta}{MM}(\epsilon,u) =\underline{y}_i - \max_{j\ne
  i} \overline{y}_i$. One can see that $\supscr{\delta}{MM}(\epsilon,u)>0$ is a sufficient condition for certified adversarial robustness. Alternatively, we can use Theorem~\ref{thm:INN} with the output transformation~\eqref{eq:underline-z} to provide less conservative lower bounds on for certified adversarial robustness of the network. We define $\delta^{\mathrm{MM-C}}(\epsilon,u)=\min_{i\in\{1,\ldots,q-1\}}\underline{z}_i^u$, where $\underline{z}^u$ is as defined in equation~\eqref{eq:underline-z}. Then, by Lemma~\ref{lem:conservative}, one can obtain the tighter sufficient condition $\delta^{\mathrm{MM-C}}(\epsilon,u)>0$ for certified adversarial robustness.

  \subsection{A simple example}\label{sec_simple_example}

  In this section, we consider a simple $2$-dimensional implicit neural network to compare different approaches for robustness
  verification. Consider an implicit neural network~\eqref{eq:INN} with
  $A=\begin{bmatrix}-\frac{1}{4} & -\frac{1}{4} \\ \frac{3}{4} &
    -\frac{1}{4}\end{bmatrix}$, $B= \begin{bmatrix}\frac{1}{2} & 1 \\
    1 & \frac{1}{2} \end{bmatrix}$, $C=I_2$, $b=c=\vect{0}_{2}$,
  and $\phi_1(\cdot)=\phi_2(\cdot)=\mathrm{ReLU}(\cdot)$. Suppose that the nominal
  input is
  $u=\begin{bmatrix}  \frac{1}{4}\\
    \frac{3}{2}\end{bmatrix}$ and due to uncertainty,
  the input is in the box
  $v\in [\underline{u},\overline{u}]$, where
  $\underline{u}=\begin{bmatrix} 0\\ 1\end{bmatrix}$ and
  $\overline{u}=\begin{bmatrix}\frac{1}{3} \\ 2\end{bmatrix}$. We
  compare the robustness bounds obtained using the Lipschitz bound approach,
  the interval bound propagation method, and our mixed monotone contractive approach. Regarding the Lipschitz
  bound approach, we use the framework in ~\cite[Corollary
  5]{SJ-AD-AVP-FB:21f} to estimate the input-output Lipschitz constant
  of the networks and thus we get $\|y(u)-y(v)\|_{\infty} \le
    \frac{\|B\|_{\infty}\|C\|_{\infty}}{1-\mu_{\infty}(A)^+}\|u-v\|_{\infty}
    = 3\|u-v\|_{\infty}$. Regarding the interval bound propagation method, using the
  iterations in~\eqref{eq:IBP-1}, we obtain $y(v) \in \left[\begin{pmatrix}  0.0342\\
      0\end{pmatrix}, \begin{pmatrix} 1.7265 \\
      2.1026\end{pmatrix}\right] $. Finally, regarding the
  mixed monotone contractive approach, using the $\alpha$-average iteration~\eqref{eq:iterations-NE} in
  Theorem~\ref{thm:INN}\ref{p4:existenceNE}, we get $y(v)\in \left[\begin{pmatrix} 0.3939 \\
      0.6364\end{pmatrix}, \begin{pmatrix}  1.6061 \\
      2.0303\end{pmatrix}\right]$. Figure~\ref{fig:compare} compares the robustness certificates
  obtained using these different approaches.

\begin{SCfigure}[25][htp]
\vspace{-.11cm}
\input{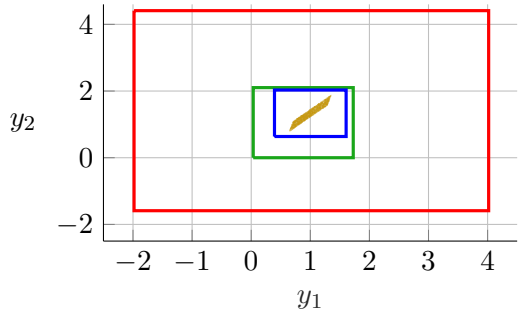}\hspace{.75cm}
\caption{Problem Setting of Section \ref{sec_simple_example}: Comparing the application of Theorem \ref{thm:INN} to existing verification methods for implicit neural networks. The yellow parallelogram shows different value of $y(u)$ for 1000 iid uniformly randomly selected $u=(u_1,u_2)^{\top}$ satisfying $0\le u_1\le \frac{1}{3}$ and $1\le u_2\le 2$. Robustness certificates attained from the Lipschitz bound approach, the interval bound propagation approach, and the application of Theorem \ref{thm:INN} are shown as red, green, and blue boxes, respectively.} \label{fig:compare}
\end{SCfigure}

\begin{figure}[!ht]
	\begin{tabular}{cc}
		\includegraphics[width = 0.482\linewidth,clip]{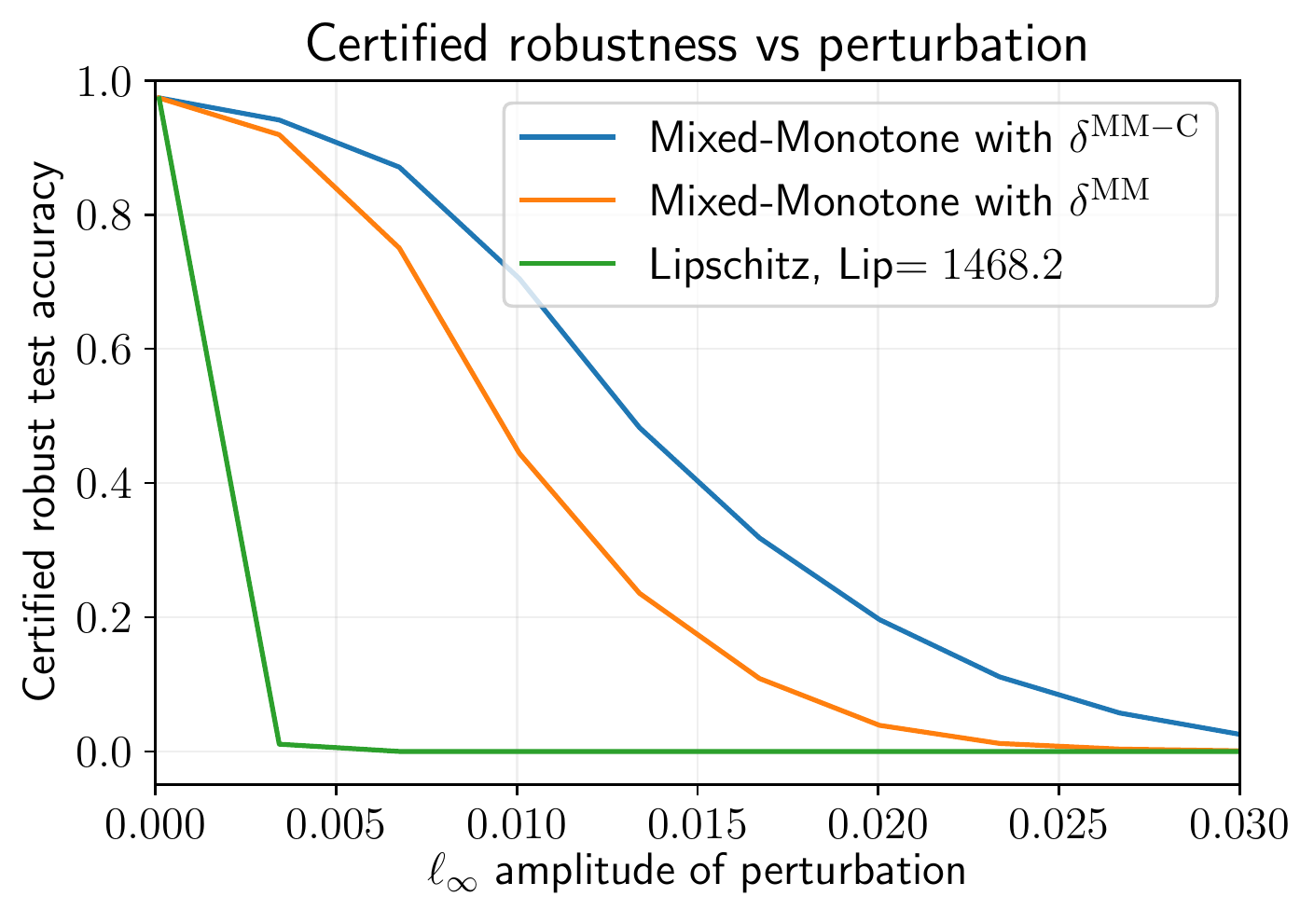}&
		\includegraphics[width = 0.482\linewidth,clip]{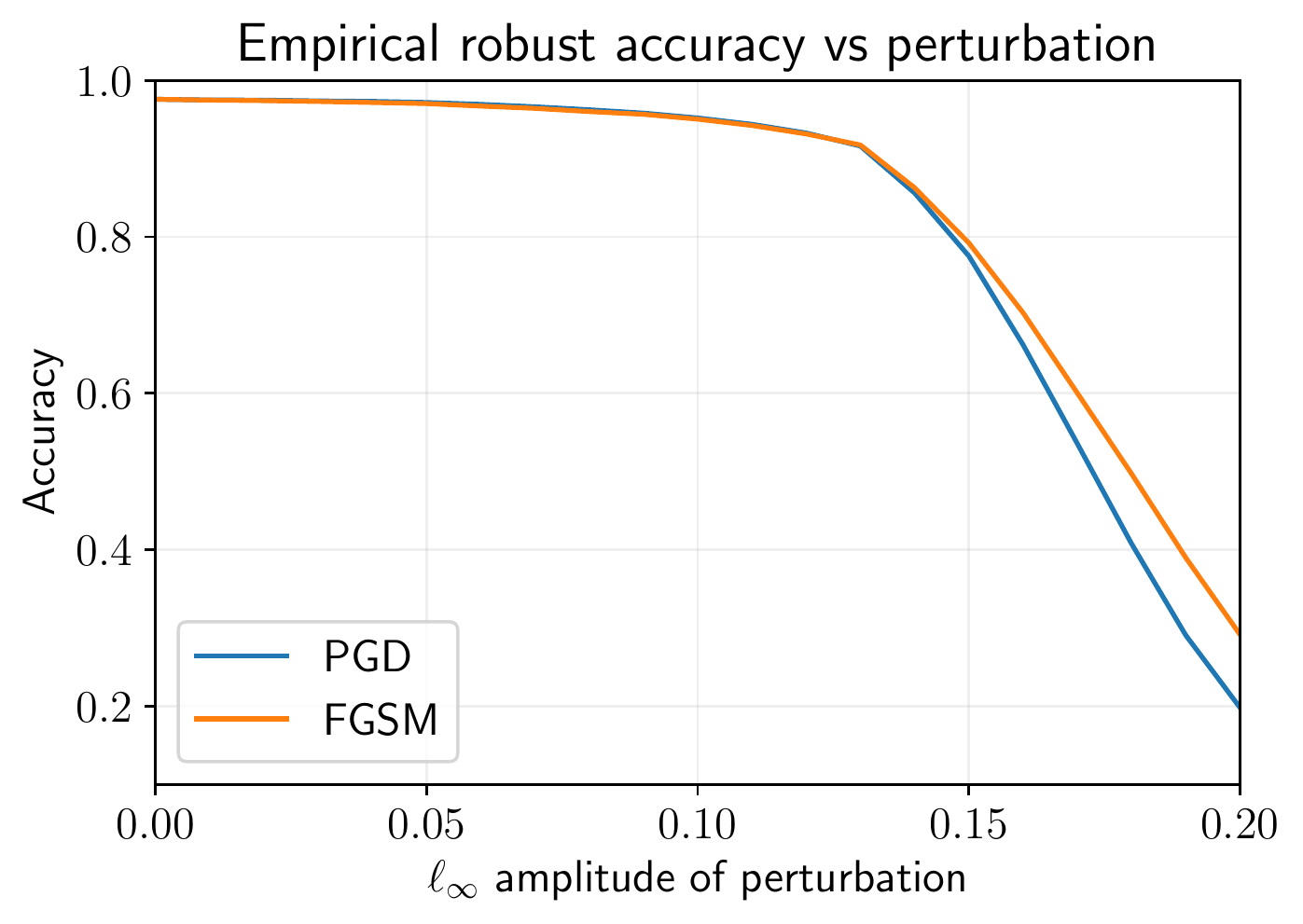}
	\end{tabular}
%\hspace{3.3cm} 
\vspace{-.35cm}
    \caption{On the left is a plot of the certified adversarial robustness of the trained NEMON model using a Lipschitz method and two mixed monotonicity methods. For fixed $\epsilon$, the fraction of test inputs which are certified robust are plotted. On the right is a plot of the empirical robustness of the same NEMON model subject to PGD and FGSM attacks. Note the difference in scale on the horizontal axis. }
    \label{fig:NEMON-MNIST}
\end{figure}
      \subsection{MNIST experiment}
      In this section, we compare the certified adversarial robustness
      of different approaches on the MNIST handwritten digit
      dataset, a dataset of $70000$ $28 \times 28$ pixel images, $60000$ of which are for training, and $10000$ for testing. Pixel values are normalized in $[0,1]$. We trained a fully-connected NEMON model, introduced in~\citep{SJ-AD-AVP-FB:21f}, with $n=100$ neurons as in the training problem~\eqref{eq:TrainingProblem}. For well-posedness, we imposed $\mu_{\infty,[\eta]^{-1}}(A) \leq 0$, where we directly parametrize the set of such $A$ as $A = [\eta]^{-1}T[\eta] - \diag(|T|\vectorones[n])$ for unconstrained $T$~\citep[Lemma 9]{SJ-AD-AVP-FB:21f}. We also use the estimate in~\citep[Corollary 5]{SJ-AD-AVP-FB:21f} for the input-output Lipschitz bound of the NEMON model. Training data was broken up into batches of $100$ and the model was trained for $15$ epochs with a learning rate of $10^{-3}$. After training, the model was validated on test data using the sufficient conditions for certified adversarial robustness in the previous section. 
      For fixed $\epsilon$ and the $10000$ test images, over $10$ trials, it took, on average, $2.250$ seconds to compute $\supscr{\delta}{Lip}(\epsilon,u)$,
      $218.099$ seconds to compute $\supscr{\delta}{IBP}(\epsilon,u)$,
      $9.087$ seconds to compute $\supscr{\delta}{MM}(\epsilon,u)$, and $11.291$ seconds to compute $\supscr{\delta}{MM-C}(\epsilon,u)$.
      To provide a conservative upper-bound on the certified adversarial robustness and to observe empirical robustness, the model was additionally attacked using projected gradient descent (PGD) and fast-gradient sign method (FGSM) attacks. Results from these experiments are shown in Figure~\ref{fig:NEMON-MNIST}. 
      
      \paragraph*{Summary evaluation.} We draw several conclusions from the experiments. First, the bounds on the certified adversarial robustness provided from the interval-bound propagation are not plotted since they provided a trivial lower bound of zero adversarial robustness for every $\epsilon$ tested. Second, we see that the bounds on the certified adversarial robustness provided by the mixed monotonicity approaches are tighter than the bounds provided by the Lipschitz constant. Third, we note the additional tightness in the bounds provided by computing the relative classifier variable $\underline{z}^u$. Finally, we observe that although mixed monotonicity approaches provide better bounds than the better-known Lipschitz and interval-bound propagation approaches, the gap between the certified robustness and the empirical robustness remains sizable, especially for larger $\ell_\infty$-perturbations.
      %that these approaches cannot seem to capture.

%      \sashatodo{Include table with runtimes for each of the methods? Or if we are limited in space, I can alternatively just put it in the text.}
%      \sabertodo{Dear Sasha, please add them to the text. I will put them in the table later.}
      
%      IBP: So possibly three issues with these iterations: 1- no convergence guarantee for our model 2-slow convergence for very small-step size 3- convergence rate changes with size of perturbation (larger epsilon gives slower convergence) + Matt please add some insight. 

\section{Conclusions}
Using mixed monotone systems theory and contraction theory, we developed a framework for studying robustness of implicit neural networks. A key tool in this approach is an embedded network that provides $\ell_{\infty}$-norm box estimates for input-output behavior of the given implicit neural network. Empirical evidence shows our approach outperforms existing methods. Future work includes (i) applying the mixed monotone contractive approach to other implicit neural networks such as MON~\citep{EW-JZK:20} and LBEN~\citep{MR-RW-IRM:20} and (ii) designing appropriate state transformations~\citep{MA-SC:21} to improve the input-output bounds in Theorem~\ref{thm:INN}.

% \begin{itemize}
%   \item Limit the main text (not counting references) to 10 PMLR-formatted pages, using this template.
%   \item Include {\em in the main text} enough details, including proof details, to convince the reviewers of the contribution, novelty and significance of the submissions.
% \end{itemize}

% Acknowledgments---Will not appear in anonymized version
% \acks{We thank a bunch of people.}

\bibliography{alias,Main,FB,New,SJ}

\begin{thebibliography}{39}
\providecommand{\natexlab}[1]{#1}
\providecommand{\url}[1]{\texttt{#1}}
\expandafter\ifx\csname urlstyle\endcsname\relax
  \providecommand{\doi}[1]{doi: #1}\else
  \providecommand{\doi}{doi: \begingroup \urlstyle{rm}\Url}\fi

\bibitem[{Abate} and {Coogan}(2020)]{MA-SC:20}
M.~{Abate} and S.~{Coogan}.
\newblock Computing robustly forward invariant sets for mixed-monotone systems.
\newblock In \emph{{IEEE} Conf.\ on Decision and Control}, pages 4553--4559,
  2020.
\newblock \doi{10.1109/CDC42340.2020.9304461}.

\bibitem[Abate and Coogan(2021)]{MA-SC:21}
M.~Abate and S.~Coogan.
\newblock Improving the fidelity of mixed-monotone reachable set approximations
  via state transformations.
\newblock In \emph{2021 American Control Conference (ACC)}, pages 4674--4679,
  2021.
\newblock \doi{10.23919/ACC50511.2021.9483264}.

\bibitem[Abate et~al.(2021)Abate, Dutreix, and Coogan]{MA-MD-SC:21}
M.~Abate, M.~Dutreix, and S.~Coogan.
\newblock Tight decomposition functions for continuous-time mixed-monotone
  systems with disturbances.
\newblock \emph{IEEE Control Systems Letters}, 5\penalty0 (1):\penalty0
  139--144, 2021.
\newblock \doi{10.1109/LCSYS.2020.3001085}.

\bibitem[Angeli and Sontag(2003)]{DA-EDS:03}
D.~Angeli and E.~D. Sontag.
\newblock Monotone control systems.
\newblock \emph{IEEE Transactions on Automatic Control}, 48\penalty0
  (10):\penalty0 1684--1698, 2003.
\newblock \doi{10.1109/TAC.2003.817920}.

\bibitem[Angeli et~al.(2014)Angeli, Enciso, and Sontag]{DA-GE-ES:14}
D.~Angeli, G.~A. Enciso, and E.~D. Sontag.
\newblock A small-gain result for orthant-monotone systems under mixed
  feedback.
\newblock \emph{Systems \& Control Letters}, 68:\penalty0 9--19, 2014.
\newblock \doi{10.1016/j.sysconle.2014.03.002}.

\bibitem[Anonymous(2022)]{anonymous:22}
Anonymous.
\newblock Certified robustness for deep equilibrium models via interval bound
  propagation.
\newblock In \emph{Submitted to The Tenth International Conference on Learning
  Representations}, 2022.
\newblock URL \url{https://openreview.net/forum?id=y1PXylgrXZ}.
\newblock under review.

\bibitem[Athalye et~al.(2018)Athalye, Engstrom, Ilyas, and
  Kwok]{AA-LE-AI-KK:18}
A.~Athalye, L.~Engstrom, A.~Ilyas, and K.~Kwok.
\newblock Synthesizing robust adversarial examples, 2018.
\newblock URL \url{https://openreview.net/forum?id=BJDH5M-AW}.

\bibitem[Bai et~al.(2019)Bai, Kolter, and Koltun]{SB-JZK-VK:19}
S.~Bai, J.~Z. Kolter, and V.~Koltun.
\newblock Deep equilibrium models.
\newblock In \emph{Advances in Neural Information Processing Systems}, 2019.
\newblock URL \url{https://arxiv.org/abs/1909.01377}.

\bibitem[Blum et~al.(2020)Blum, Dick, Manoj, and Zhang]{AB-TD-NM-HZ:20}
A.~Blum, T.~Dick, N.~Manoj, and H.~Zhang.
\newblock Random smoothing might be unable to certify $\ell_\infty$ robustness
  for high-dimensional images.
\newblock \emph{Journal of Machine Learning Research}, 21\penalty0
  (211):\penalty0 1--21, 2020.
\newblock URL \url{http://jmlr.org/papers/v21/20-209.html}.

\bibitem[Bullo et~al.(2021)Bullo, Cisneros-Velarde, Davydov, and
  Jafarpour]{FB-PCV-AD-SJ:21e}
F.~Bullo, P.~Cisneros-Velarde, A.~Davydov, and S.~Jafarpour.
\newblock From contraction theory to fixed point algorithms on {Riemannian} and
  non-{Euclidean} spaces.
\newblock In \emph{{IEEE} Conf.\ on Decision and Control}, December 2021.
\newblock URL \url{https://arxiv.org/pdf/2110.03623}.
\newblock To appear (Invited Tutorial Session).

\bibitem[Carlini and Wagner(2017)]{NC-DW:17}
N.~Carlini and D.~Wagner.
\newblock Adversarial examples are not easily detected: {B}ypassing ten
  detection methods.
\newblock In \emph{ACM Workshop on Artificial Intelligence and Security}, pages
  3--14, 2017.
\newblock \doi{10.1145/3128572.3140444}.

\bibitem[Chen et~al.(2021)Chen, Lasserre, Magron, and Pauwels]{TC-JBL-VM-EP:21}
T.~Chen, J.~B. Lasserre, V.~Magron, and E.~Pauwels.
\newblock Semialgebraic representation of monotone deep equilibrium models and
  applications to certification.
\newblock In \emph{Thirty-Fifth Conference on Neural Information Processing
  Systems}, 2021.
\newblock URL \url{https://openreview.net/forum?id=m4rb1Rlfdi}.

\bibitem[Cohen et~al.(2019)Cohen, Rosenfeld, and Kolter]{JC-ER-JZK:19}
J.~Cohen, E.~Rosenfeld, and {J. Z.} Kolter.
\newblock Certified adversarial robustness via randomized smoothing.
\newblock In \emph{Int.\ Conf.\ on Machine Learning}, pages 1310--1320, 2019.
\newblock URL \url{https://arxiv.org/abs/1902.02918}.

\bibitem[Combettes and Pesquet(2020)]{LPC-JCP:20}
P.~L. Combettes and J-C. Pesquet.
\newblock Lipschitz certificates for layered network structures driven by
  averaged activation operators.
\newblock \emph{SIAM Journal on Mathematics of Data Science}, 2\penalty0
  (2):\penalty0 529--557, 2020.
\newblock \doi{10.1137/19M1272780}.

\bibitem[Coogan(2020)]{SC:20}
S.~Coogan.
\newblock Mixed monotonicity for reachability and safety in dynamical systems.
\newblock In \emph{2020 59th IEEE Conference on Decision and Control (CDC)},
  pages 5074--5085, 2020.
\newblock \doi{10.1109/CDC42340.2020.9304391}.

\bibitem[Coogan and Arcak(2015)]{SC-MA:15b}
S.~Coogan and M.~Arcak.
\newblock Efficient finite abstraction of mixed monotone systems.
\newblock In \emph{Hybrid Systems: Computation and Control}, pages 58--67,
  April 2015.
\newblock \doi{10.1145/2728606.2728607}.

\bibitem[Davydov et~al.(2021)Davydov, Jafarpour, and Bullo]{AD-SJ-FB:20o}
A.~Davydov, S.~Jafarpour, and F.~Bullo.
\newblock {Non-Euclidean} contraction theory for robust nonlinear stability.
\newblock \emph{IEEE Transactions on Automatic Control}, July 2021.
\newblock URL \url{https://arxiv.org/abs/2103.12263}.
\newblock Submitted.

\bibitem[{El~Ghaoui} et~al.(2021){El~Ghaoui}, Gu, Travacca, Askari, and
  Tsai]{LEG-FG-BT-AA-AYT:21}
L.~{El~Ghaoui}, F.~Gu, B.~Travacca, A.~Askari, and A.~Tsai.
\newblock Implicit deep learning.
\newblock \emph{SIAM Journal on Mathematics of Data Science}, 3\penalty0
  (3):\penalty0 930--958, 2021.
\newblock \doi{10.1137/20M1358517}.

\bibitem[Enciso et~al.(2006)Enciso, Smith, and Sontag]{GE-HS-ES:06}
G.~A. Enciso, H.~L. Smith, and E.~D. Sontag.
\newblock Nonmonotone systems decomposable into monotone systems with negative
  feedback.
\newblock \emph{Journal of Differential Equations}, 224\penalty0 (1):\penalty0
  205--227, 2006.
\newblock \doi{10.1016/j.jde.2005.05.007}.

\bibitem[Farina and Rinaldi(2000)]{LF-SR:00}
L.~Farina and S.~Rinaldi.
\newblock \emph{Positive Linear Systems: Theory and Applications}.
\newblock John Wiley \& Sons, 2000.
\newblock ISBN 0471384569.

\bibitem[Fazlyab et~al.(2019)Fazlyab, Robey, Hassani, Morari, and
  Pappas]{MF-AR-HH-MM-GJP:19}
M.~Fazlyab, A.~Robey, H.~Hassani, M.~Morari, and G.~J. Pappas.
\newblock Efficient and accurate estimation of {L}ipschitz constants for deep
  neural networks.
\newblock In \emph{Advances in Neural Information Processing Systems}, 2019.
\newblock URL \url{https://arxiv.org/abs/1906.04893}.

\bibitem[Goodfellow et~al.(2015)Goodfellow, Shlens, and Szegedy]{IJG-JS-CZ:15}
I.~J. Goodfellow, J.~Shlens, and C.~Szegedy.
\newblock Explaining and harnessing adversarial examples.
\newblock In \emph{International Conference on Learning Representations
  (ICLR)}, 2015.
\newblock URL \url{https://arxiv.org/abs/1412.6572}.

\bibitem[Gowal et~al.(2018)Gowal, Dvijotham, Stanforth, Bunel, Qin, Uesato,
  Arandjelovic, Mann, and Kohli]{SG-etal:18}
S.~Gowal, K.~Dvijotham, R.~Stanforth, R.~Bunel, C.~Qin, J.~Uesato,
  R.~Arandjelovic, T.~Mann, and P.~Kohli.
\newblock On the effectiveness of interval bound propagation for training
  verifiably robust models.
\newblock \emph{arXiv preprint arXiv:1810.12715}, 2018.

\bibitem[Jafarpour et~al.(2021)Jafarpour, Davydov, Proskurnikov, and
  Bullo]{SJ-AD-AVP-FB:21f}
S.~Jafarpour, A.~Davydov, A.~V. Proskurnikov, and F.~Bullo.
\newblock Robust implicit networks via non-{Euclidean} contractions.
\newblock In \emph{Advances in Neural Information Processing Systems}, December
  2021.
\newblock URL \url{http://arxiv.org/abs/2106.03194}.

\bibitem[Kag et~al.(2020)Kag, Zhang, and Saligrama]{AK-ZZ-VS:20}
A.~Kag, Z.~Zhang, and V.~Saligrama.
\newblock {RNNs} incrementally evolving on an equilibrium manifold: {A} panacea
  for vanishing and exploding gradients?
\newblock In \emph{International Conference on Learning Representations}, 2020.
\newblock URL \url{https://openreview.net/forum?id=HylpqA4FwS}.

\bibitem[Li et~al.(2019)Li, Chen, Wang, and Carin]{BL-CC-WW-LC:19}
B.~Li, C.~Chen, W.~Wang, and L.~Carin.
\newblock Certified adversarial robustness with additive noise.
\newblock In \emph{Advances in Neural Information Processing Systems}, 2019.
\newblock URL \url{https://arxiv.org/abs/1809.03113}.

\bibitem[Madry et~al.(2018)Madry, Makelov, Schmidt, Tsipras, and
  Vladu]{AM-AM-LS-DT-AV:17}
A.~Madry, A.~Makelov, L.~Schmidt, D.~Tsipras, and A.~Vladu.
\newblock Towards deep learning models resistant to adversarial attacks.
\newblock In \emph{International Conference on Machine Learning}, 2018.
\newblock URL \url{https://arxiv.org/abs/1706.06083}.

\bibitem[Mirman et~al.(2018)Mirman, Gehr, and Vechev]{MM-TG-MV:18}
M.~Mirman, T.~Gehr, and M.~Vechev.
\newblock Differentiable abstract interpretation for provably robust neural
  networks.
\newblock In J.~Dy and A.~Krause, editors, \emph{Proceedings of the 35th
  International Conference on Machine Learning}, volume~80 of \emph{Proceedings
  of Machine Learning Research}, pages 3578--3586. PMLR, 10-15 Jul 2018.
\newblock URL \url{https://proceedings.mlr.press/v80/mirman18b.html}.

\bibitem[Pabbaraju et~al.(2021)Pabbaraju, Winston, and Kolter]{CP-EW-JZK:21}
C.~Pabbaraju, E.~Winston, and J.~Z. Kolter.
\newblock Estimating {L}ipschitz constants of monotone deep equilibrium models.
\newblock In \emph{International Conference on Learning Representations}, 2021.
\newblock URL \url{https://openreview.net/forum?id=VcB4QkSfyO}.

\bibitem[Papernot et~al.(2016)Papernot, McDaniel, Wu, Jha, and
  Swami]{NP-PM-XW-SJ-AS:16}
N.~Papernot, P.~McDaniel, X.~Wu, S.~Jha, and A.~Swami.
\newblock Distillation as a defense to adversarial perturbations against deep
  neural networks.
\newblock In \emph{IEEE Symposium on Security and Privacy (SP)}, pages
  582--597, 2016.
\newblock \doi{10.1109/SP.2016.41}.

\bibitem[Revay et~al.(2020)Revay, Wang, and Manchester]{MR-RW-IRM:20}
M.~Revay, R.~Wang, and I.~R. Manchester.
\newblock Lipschitz bounded equilibrium networks.
\newblock 2020.
\newblock URL \url{https://arxiv.org/abs/2010.01732}.

\bibitem[Smith(1995)]{HLS:95}
H.~L. Smith.
\newblock \emph{Monotone Dynamical Systems: An Introduction to the Theory of
  Competitive and Cooperative Systems}.
\newblock American Mathematical Society, 1995.
\newblock ISBN 082180393X.

\bibitem[Szegedy et~al.(2014)Szegedy, Zaremba, Sutskever, Bruna, Erhan,
  Goodfellow, and Fergus]{CZ-WZ-IS-JB-DE-IG-RF:13}
C.~Szegedy, W.~Zaremba, I.~Sutskever, J.~Bruna, D.~Erhan, I.~Goodfellow, and
  R.~Fergus.
\newblock Intriguing properties of neural networks.
\newblock In \emph{International Conference on Learning Representations}, 2014.
\newblock URL \url{https://arxiv.org/abs/1312.6199}.

\bibitem[Virmaux and Scaman(2018)]{AV-KS:18}
A.~Virmaux and K.~Scaman.
\newblock Lipschitz regularity of deep neural networks: analysis and efficient
  estimation.
\newblock In \emph{Advances in Neural Information Processing Systems},
  volume~31, page 3839–3848, 2018.
\newblock URL
  \url{https://proceedings.neurips.cc/paper/2018/file/d54e99a6c03704e95e6965532dec148b-Paper.pdf}.

\bibitem[Winston and Kolter(2020)]{EW-JZK:20}
E.~Winston and J.~Z. Kolter.
\newblock Monotone operator equilibrium networks.
\newblock In \emph{Advances in Neural Information Processing Systems}, 2020.
\newblock URL \url{https://arxiv.org/abs/2006.08591}.

\bibitem[Wong and Kolter(2018)]{EW-ZK:18}
E.~Wong and J.~Z. Kolter.
\newblock Provable defenses against adversarial examples via the convex outer
  adversarial polytope.
\newblock In \emph{International Conference on Machine Learning}, pages
  5286--5295, 2018.
\newblock URL \url{http://proceedings.mlr.press/v80/wong18a.html}.

\bibitem[Yang and Ozay(2019)]{LY-NO:19}
L.~Yang and N.~Ozay.
\newblock Tight decomposition functions for mixed monotonicity.
\newblock In \emph{2019 IEEE 58th Conference on Decision and Control (CDC)},
  pages 5318--5322, 2019.
\newblock \doi{10.1109/CDC40024.2019.9030065}.

\bibitem[Zhang et~al.(2018)Zhang, Weng, Chen, Hsieh, and Daniel]{HZ-etal:18}
H.~Zhang, T-W. Weng, P-Y. Chen, C-J. Hsieh, and L.~Daniel.
\newblock Efficient neural network robustness certification with general
  activation functions.
\newblock In \emph{Advances in Neural Information Processing Systems}, page
  4944–4953, 2018.
\newblock URL \url{https://arxiv.org/abs/1811.00866}.

\bibitem[Zhang et~al.(2020)Zhang, Chen, Xiao, Gowal, Stanforth, Li, Boning, and
  Hsieh]{HZ-etal:20}
H.~Zhang, H.~Chen, C.~Xiao, S.~Gowal, R.~Stanforth, Bo~Li, D.~Boning, and C-J.
  Hsieh.
\newblock Towards stable and efficient training of verifiably robust neural
  networks.
\newblock In \emph{International Conference on Learning Representations}, 2020.
\newblock URL \url{https://openreview.net/forum?id=Skxuk1rFwB}.

\end{thebibliography}

\appendix

\end{document}